\pdfoutput=1
\documentclass[pdflatex,sn-basic]{sn-jnl}

\usepackage{amsmath, amssymb, amsthm}
\usepackage{geometry}
\usepackage{bm}
\usepackage{hyperref}

\newtheorem{definition}{Definition}
\newtheorem{theorem}{Theorem}
\newtheorem{lemma}{Lemma}

\newtheorem{remark}{Remark}

\title{
Interpretation as Linear Transformation: A Cognitive-Geometric Model of Concepts and Meaning
}

\author*[1]{Chainarong Amornbunchornvej}

\affil*[1]{National Electronics and Computer Technology Center,
112 Phahonyothin Road, Khlong Nueng, Pathum Thani, 12120, Thailand\\
\email{chainarong.amo@nectec.or.th}}

\begin{document}

\maketitle

\begin{abstract}

 This paper develops a geometric framework for modeling concepts, motivation, and
influence across cognitively heterogeneous agents. Each agent is represented by
a personalized \emph{value space}, a vector space encoding the internal
dimensions through which the agent interprets and evaluates meaning. Evaluative
concepts are formalized as structured vectors---\emph{abstract beings}---whose
transmission is mediated by linear \emph{interpretation maps}. An abstract being survives
communication only if it avoids the null spaces of these maps, yielding a
structural criterion for intelligibility, miscommunication, and concept death.

Within this framework, I show how conceptual distortion, motivational drift,
and the limits of mutual understanding arise from
purely algebraic constraints. A central result---the \emph{No-Null-Space
Leadership Condition}---characterizes leadership as a property of
representational reachability rather than persuasion or authority. More
broadly, the model explains how abstract beings can propagate, mutate, or
disappear as they traverse diverse cognitive geometries.

The account unifies insights from conceptual spaces, social epistemology, and
AI value alignment by grounding meaning preservation in structural
compatibility rather than shared information or rationality. I argue that this
cognitive-geometric perspective clarifies the epistemic boundaries of influence
in both human and artificial systems, and offers a general foundation for
analyzing conceptual dynamics across heterogeneous agents.
\end{abstract}

\section{Introduction}

Much of human and artificial communication is modeled as if concepts and beliefs were
propositional items that can be transmitted intact from one mind to another.
On this view, misunderstanding arises primarily from epistemic failures:
incorrect beliefs, noisy signals, biased inference, or insufficient data.
Yet many cases of miscommunication and failed influence persist even when all
such epistemic factors are idealized. Agents may share evidence, speak
truthfully, and reason flawlessly, and still fail to understand one another.

This suggests that the fundamental barrier is not epistemic but
\emph{structural}. Agents differ in how they carve the world into dimensions of
meaning, how they evaluate alternatives, and how they encode motivational
significance. Communication therefore involves not the copying of content but
its \emph{transformation} into a new representational basis. A message that
travels between cognitively heterogeneous agents may be rotated, distorted,
selectively amplified, or annihilated depending on the geometry of the
recipient's evaluative space.

To make this structural dimension explicit, the paper introduces a cognitive–
geometric framework in which each agent possesses a personalized \emph{value
space}: a finite-dimensional vector space whose basis encodes the agent's
interpretive and evaluative structure. Concepts and goals are represented as
structured vectors—\emph{abstract beings}—that acquire meaning only relative to
an agent's basis. Communication from agent $A$ to agent $B$ is modeled by a
linear \emph{interpretation map} $T_{A \to B}$ that projects $A$'s abstract
beings into $B$'s value space. An abstract being survives this projection only if it
avoids the relevant null spaces, which determine the components that are
invisible or unintelligible to the recipient.

Throughout the paper, I use ``concept'' as the general term for abstract beings, and use ``belief'' informally for concepts that an agent has endorsed.

This geometric perspective provides unified explanations for a range of
philosophical phenomena: partial understanding, meaning drift, value
misalignment, motivational divergence, and the extinction of concepts within
populations. The core claim is that these arise from algebraic constraints on
representation rather than from irrationality or limited evidence. Even ideally
rational agents may systematically disagree or misunderstand one another
because their value spaces differ.

The framework also illuminates dynamic phenomena. As agents adopt new abstract
beings, their goal vectors and motivational gradients shift, allowing for
systematic modeling of goal drift, alignment with influential agents, and the
emergence of shared identities. Conversely, when an innovator holds a value
vector lying outside the convex hull of a group's current valuations, they can
introduce genuinely new cultural or conceptual directions that the group could
not generate endogenously.

Although leadership provides a vivid application, it is only one illustration
of the broader theoretical machinery. The model also explains
marketing-induced reconfiguration of value spaces.

The next subsection formalizes these insights by summarizing the paper's theoretical contributions and positioning them within existing philosophical and cognitive–scientific frameworks.

\subsection*{Theoretical Contributions}

This paper makes four contributions to the theory of conceptual representation,
interpretation, and influence.

\paragraph{(1) A geometric ontology of evaluative concepts.}
The framework models each agent's evaluative structure as a value space and
treats concepts as structured vectors—abstract beings—defined relative to an
agent's basis. This shifts the analysis of meaning from propositional content to
transformational structure, revealing that many forms of misunderstanding and
partial agreement arise from representational geometry rather than epistemic
error. Beliefs emerge as a special case: concepts that are endorsed and motivationally active within an agent's space.

\paragraph{(2) A structural account of intelligibility and influence.}
Communication is formalized as a linear interpretation map between agents'
value spaces. Whether a concept survives transmission depends on whether it
avoids the null spaces of these maps. This yields precise geometric criteria
for intelligibility, miscommunication, conceptual distortion, and concept death, and
establishes \emph{structural reachability} as a precondition for influence.

\paragraph{(3) A unified model of motivational dynamics.}
By treating motivational gradients as vectors in the same space as concepts, the
framework explains how adopting an abstract being alters an agent's goals and
action tendencies. This supports formal accounts of goal drift, value
convergence, and alignment with leaders or influential agents.

\paragraph{(4) A model of innovation and value-space expansion.}
The framework explains how genuinely novel evaluative directions arise when a
leader or innovator occupies a vector lying outside the convex hull of group
valuations. Such cases expand the group's representational geometry rather than
interpolate within it, offering a geometric account of conceptual or cultural
innovation.

Together, these contributions provide a unified cognitive–geometric foundation
for studying conceptual dynamics in heterogeneous populations. They clarify the
conditions under which meaning can be preserved, distorted, or lost across
agents, and they illuminate the epistemic boundaries of influence in both human
and artificial systems.

\section{Conceptual Foundations: Agents, Concepts, and Interpretive Geometry}
\label{sec:ConceptualFoundations}

Before developing the formal model, I outline the conceptual assumptions that motivate
the geometric approach. This section provides an intuitive bridge between philosophical
concerns about meaning, understanding, and value alignment, and the mathematical
structures introduced later.

\subsection{Agents as Value Spaces}

I assume that each agent possesses a \emph{personalized value space} $\mathcal{V}_A$, a vector
space encoding the dimensions along which the agent represents and evaluates meaning.
These dimensions may correspond to moral priorities, pragmatic considerations,
identity-related concerns, or epistemic preferences. This assumption aligns with
geometric models of cognition and conceptual spaces \cite{gardenfors2004conceptual,
gardenfors2014geometry}, which treat representational content as structured regions of
a cognitive geometry.

An evaluative concept for agent $A$ is represented as a vector $X_A \in \mathcal{V}_A$. Such vectors, which I
call \emph{abstract beings}, carry both semantic and motivational structure. Their norm
reflects subjective importance; their direction reflects the internal organization of the
concept.

\subsection{Interpretation as Transformative Mapping}
\label{sec:InterpretationMapping}

Understanding between agents is not assumed to be direct. Instead, when agent $A$
expresses a concept to agent $B$, the concept undergoes an interpretive transformation
\[
T_{A \rightarrow B} : \mathcal{V}_A \to \mathcal{V}_B,
\]
which projects $A$'s representational structure into $B$'s cognitive geometry. The map
may distort or omit components of the concept. Concepts that fall into the null space of
$T_{A \rightarrow B}$ become unintelligible to $B$, providing a structural account of
misunderstanding, interpretive opacity, or motivational loss.

This formalization connects to Dennett's intentional stance \cite{dennett1987intentional} in a precise way. On Dennett's account, interpreting an agent means attributing beliefs and desires that make the agent's behavior predictable and rational. In Cognitive Geometry, this attribution corresponds to selecting an interpretation map $T_{A \to B}$ that renders $A$'s abstract beings coherent within $B$'s value space. Forward and backward consistency conditions (formalized in Appendix~\ref{apx:InMap}) capture what it means for such an attribution to be stable: $B$'s interpretation of $A$'s concept must be approximately recoverable by $A$, and vice versa. When no such map preserves the concept's structure, the intentional stance breaks down—not because the agent is irrational, but because the interpretive geometry is incompatible. The intentional stance is thus rendered algebraically: understanding is a function of mutual interpretability across cognitive geometries.

\subsection{Concept Dynamics as Geometric Survivability}

Communication and influence are modeled as the survivability of concept vectors under
interpretation. A concept survives transmission when its transformed representation
retains sufficient magnitude and structural integrity. Failures of persuasion,
miscommunication, ideological mutation, or loss of meaning arise naturally from this
framework: they correspond to geometric degradation under repeated transformations.

This geometric notion of survivability parallels concerns in AI alignment, where an
agent's internal representation of a goal may diverge from the human's intended meaning
\cite{gabriel2020artificial}. My approach provides a formal
criterion for such divergences.

\subsection{Why Abstraction Is Necessary}

Traditional models of communication often presuppose shared semantic structure or
universal cognitive priors. But real systems—human or artificial—exhibit heterogeneous
representational schemes. The geometric formalism introduced here allows us to:

\begin{itemize}
    \item analyze compatibility between cognitive architectures,
    \item formalize the conditions under which concepts preserve meaning across agents,
    \item model failures of leadership or coordination arising from structural, not
          informational, barriers.
\end{itemize}

This conceptual foundation motivates the formal model developed in the following
sections, where value spaces, interpretive maps, null spaces, and concept survivability
are given precise mathematical form.

\subsection{Related Work}
\label{sec:RelatedWork}

This work builds upon and synthesizes several distinct research trajectories across artificial intelligence, cognitive science, and formal epistemology.

\paragraph{Conceptual Spaces and Geometric Cognition.}
My notion of value spaces extends the idea of \textit{conceptual spaces} as developed by Gärdenfors \cite{gardenfors2004conceptual}, which represent mental content using geometric structures. Whereas Gärdenfors' Geometry of Meaning~\cite{gardenfors2014geometry} focuses on single-agent conceptual semantics, my model introduces multi-agent value spaces and interpretive operators that explain disagreement, motivation, and influence. In deep learning, work by~\cite{geva2021transformer} demonstrates that Transformer models encode semantic information through vector-space structures, such as key–value memories and semantic patterns concentrated in upper layers, with subsequent work~\cite{pmlr-v235-park24c} formalizing how high-level concepts are represented as linearly composed feature directions, thereby reinforcing geometric interpretations of meaning. \cite{saxe2019mathematical} provide a mathematical analysis of how representations in deep linear networks evolve during training, showing that the learning dynamics follow predictable geometric trajectories. \cite{piantadosi2024concepts} extend this observation to cognitive representations more broadly, arguing that vector-based representations can capture the compositional and structural properties of human concepts, with linear operations supporting key relational and analogical computations.

\paragraph{Linearity as an Approximation for Evaluative Structures.}
The assumption of linearity is motivated by converging evidence across multiple domains. In distributional semantics, semantic relationships are well-captured by linear operations in embedding spaces: word analogies, for instance, correspond to vector arithmetic~\cite{mikolov2013linguistic}, and cross-lingual alignment can be achieved by learning linear maps between separately trained embedding spaces~\cite{conneau2017word, schonemann1966generalized}. However, the present framework applies linearity not only to semantic or perceptual content but also to evaluative and motivational structures—valuations, desires, and emotional appraisals. The justification here is that linearity serves as a principled first-order approximation: any smooth mapping between cognitive spaces can be locally approximated by its Jacobian (a linear map) in a neighborhood of the current operating point, much as Taylor expansion linearizes nonlinear functions. This first-order approximation captures the dominant effects of interpretive transformation—projection, distortion, and annihilation—while remaining analytically tractable. The framework's core results (null-space filtering, compositional reachability, leadership conditions) depend on the algebraic structure of linear maps; an informal piecewise-linear extension of the Local Coherence Theorem is provided in Appendix~\ref{apx:LocalCoherence}, and a fully nonlinear treatment remains a direction for future work (see Limitations).

\paragraph{Concept Modeling and Communication.}  
While Levesque's framework formalizes belief and agency in a logical, symbolic manner~\cite{levesque1984logic}, my model reformulates these ideas in geometric terms. In the realm of communication, \cite{clark1996using} emphasizes that successful interaction requires establishing common ground—a shared base of knowledge and context that speakers continuously coordinate and update. Similarly, \cite{tomasello2005shared} argue that human communication is fundamentally built on shared intentionality and collaborative joint action rather than mere information exchange through signals. My model provides a formal geometric interpretation of these insights: miscommunication occurs when a transmitted concept vector falls into the null space of the receiver's interpretation map, making genuine grounding impossible in structural terms. In the cognitive geometry developed here, shared intentions are represented as aligned basis vectors: dimensions jointly encoded across agents' value spaces.

Recent work by~\citet{wheeler2023semantic} applies conceptual spaces to semantic communication, modeling shared understanding in terms of alignment within structured semantic spaces. My framework shares this geometric orientation but differs in a key respect: whereas Wheeler et al.\ analyze alignment within a common conceptual space, the present model emphasizes communication \emph{across} distinct value spaces with potentially incompatible bases, and provides a null-space criterion that determines when alignment is structurally impossible rather than merely incomplete.

Moreover, formal results linking symbolic semantics to vector-space representations show that extensional logical semantics can be systematically reformulated in geometric terms while preserving compositional structure on the relevant fragment~\cite{quigley2025vector}.

\paragraph{Value Alignment and AI Interpretability.}  
My emphasis on value-preserving transformation relates directly to the literature on AI alignment and interpretability. \cite{gabriel2020artificial} situates value alignment within broader ethical frameworks, focusing on how AI systems should reflect human moral commitments, while \cite{russell2019human} approaches alignment through normative principles and empirical failure cases, focusing on how AI systems ought to reflect human values in practice. While these works focus on normative or empirical criteria for value alignment, my model aims at structural interpretability: explaining alignment and misalignment in terms of the geometric organization of conceptual representations.

\paragraph{Multi-Agent Influence and Leadership.}  
Classical models of influence in networks often rely on probabilistic diffusion (e.g., \cite{jackson2007diffusion}), whereas my model incorporates structural interpretability via null spaces. Classical diffusion models (e.g., \cite{kempe2003maximizing}) treat influence as probabilistic contagion, whereas my model introduces structural representational constraints absent in such frameworks. My approach shares their spirit but prioritizes epistemic compatibility over metric learning.

\paragraph{Ontology, Symbol Grounding, and Philosophical Framing.}  
My approach also speaks to long-standing debates on mental content and grounding~\cite{harnad1990symbol}, though via formal rather than semantic mechanisms. The notion of \textit{abstract beings} relates to Floridi's ontological treatment of informational objects~\cite{floridi2011information}. In Floridi's philosophy of information, informational objects are accessed through a \emph{Level of Abstraction}—a set of typed observables that determines what features of an entity are registered. Abstract beings in my framework play an analogous role: they are defined relative to an agent's basis, and their meaning is constituted by their position within a value space and their transformational behavior under interpretation maps, rather than by correspondence to external reality. The challenge of concept death connects to work on collective memory and cultural evolution: in my framework, collective memory is represented as the cognitive space spanned by all agents' conceptual structures, and abstract entities "die" when they fall into the null space of every agent's interpretive map.

\paragraph{Philosophical Foundations: Grounding, Understanding, and the Intentional Stance.} 
My treatment of concepts as abstract beings echoes longstanding ontological concerns in the philosophy of mind. Dennett's intentional stance \cite{dennett1987intentional} frames interpretation as a matter of predicting an agent's behavior through coherent attributions of belief and desire; Section~\ref{sec:InterpretationMapping} develops the formal connection between the intentional stance and interpretation maps in detail.

My notion of concept death aligns with Longino's emphasis on plurality and perspectival divergence \cite{longino1990science}: deep disagreement can reflect structural incompatibility of value spaces rather than epistemic failure.

\paragraph{Cognitive Structural Representation in Neuroscience.} 
Representational geometry provides a unifying framework for understanding cognitive and neural representations, as articulated by~\cite{kriegeskorte2013representational}. Rather than encoding information in individual neurons, this framework emphasizes the geometric structure of population-level response patterns, a principle supported by a broad range of empirical findings surveyed in their work.

In a complementary theoretical vein, \cite{wei2025representational} show that systematic behavioral error patterns can be explained by geometric constraints on internal representations under information-processing limits, highlighting the explanatory role of representational geometry at the cognitive level.

More directly, \cite{greco2024predictive} provide empirical evidence that learning and prediction reorganize the geometry of neural population activity, demonstrating that task-relevant information is encoded in population-level response structures rather than in isolated neurons.

Building on this foundation, the present model extends representational geometry beyond perceptual coding to the domain of evaluative concepts and value, treating agents' internal states as vectors in personalized value spaces and modeling communication as linear transformations between these spaces.

\medskip

Nevertheless, none of these frameworks provide a \textit{structural criterion for the survivability of concepts under heterogeneous cognitive architectures}. Conceptual-space models characterize semantic similarity and grounding geometrically, but typically presuppose shared or comparable dimensions and do not specify when a concept becomes structurally unintelligible to another agent. Representational-geometry approaches in neuroscience explain how information is encoded and transformed within population-level activity, yet they remain largely confined to perceptual or task-specific representations and do not address conceptual transmission or motivational influence across agents with distinct evaluative bases. Logical and symbolic models of belief, in turn, formalize inferential relations but abstract away from the representational transformations that occur during communication.

The present framework integrates these strands by introducing a \textit{null-space criterion of concept survivability}: a concept can influence another agent if and only if its representation avoids annihilation under the relevant interpretation maps. This yields a mathematically explicit, geometrically interpretable, and cognitively grounded account of conceptual transmission, miscommunication, and motivational alignment across heterogeneous agents.

\paragraph{What are Abstract Beings?}
I define \emph{abstract beings} as internal constructs that represent evaluative, value-oriented entities within an agent's cognitive space. These may correspond to ideologies, goals, narratives, identities, or hypothetical propositions—anything that carries evaluative weight and can exert motivational force. Unlike external objects or percepts, abstract beings have no necessary grounding in physical reality. They are geometrically represented as vectors in an agent's value space. An abstract being functions as a \emph{belief} when it is endorsed by an agent.

\emph{Intuitively}, one may think of an abstract being like a ``constellation in the mind''—not a real star cluster, but a meaningful pattern one perceives, follows, and even fights for. A political ideal (e.g., ``freedom'') or a personal identity (e.g., ``I am a protector'') operates as such a being. These entities persist, evolve, and die, depending on how they are represented, interpreted, and propagated across agents.
\section{Theory}

\subsection{Value Spaces and Cognitive Motivation}

I begin by defining the internal structure of an agent as a personalized \textit{value space}, a vector space that encodes what the agent considers meaningful or valuable. The basis vectors of this space represent the agent’s fundamental evaluative dimensions.

Each agent also possesses a \textit{valuation function}, which assigns significance to any concept vector within its space. This allows agents not only to hold concepts, but to evaluate them.

Motivation is then formalized as a vector pointing from the agent’s current state to a goal state. This motivational gradient determines the direction in which an agent is inclined to act. As such, cognition and desire are unified within a single geometric framework. In this language of linear algebra, I can show how two people have different perspectives on the same goal or other things. See Appendix~\ref{apx:ValMot} for the formalization.

\subsubsection{Example: Divergent Valuation via Perspective Transformation}

Consider two individuals evaluating an abstract goal \( X \), such as publishing a philosophical book. Let the shared value space have three dimensions: personal fulfillment (basis vector \( b_1 \)), social recognition (\( b_2 \)), and financial reward (\( b_3 \)). In Person A’s coordinate system, the value vector of \( X \) is:

\[
\mathbf{x}_A =
\begin{bmatrix}
0.9 \\ 0.6 \\ 0.2
\end{bmatrix}
\]

This indicates high fulfillment, moderate recognition, and low monetary return. Person A’s value basis is assumed to be identity.

Person B, however, values financial reward more heavily and fulfillment less so. Their perspective is modeled by the linear transformation:

\[
T_{A \rightarrow B} =
\begin{bmatrix}
0.3 & 0 & 0 \\
0 & 0.5 & 0 \\
0 & 0 & 1.2
\end{bmatrix}
\]

Applying this yields:

\[
\mathbf{x}_B = T_{A \rightarrow B} \cdot \mathbf{x}_A =
\begin{bmatrix}
0.27 \\ 0.3 \\ 0.24
\end{bmatrix}
\]

To simplify things, let the valuation function be a simple sum of vector components:

\[
\mathrm{Val}_A(\mathbf{x}_A) = 1.7, \quad \mathrm{Val}_B(\mathbf{x}_B) = 0.81.
\]

Although both assess the same goal, Person B values it less than half as much due to their differing internal value system. This illustrates how personal perspective—modeled as a linear transformation—alters perceived value. 

\subsubsection{Example: Motivation as Gradient Between States}

Let an agent's value space be defined by three dimensions: knowledge (\( b_1 \)), reputation (\( b_2 \)), and income (\( b_3 \)).

Suppose the agent's \textbf{current state} is:

\[
x_i =
\begin{bmatrix}
0.4 \\ 0.3 \\ 0.2
\end{bmatrix}
\quad \text{and their desired \textbf{goal state} is} \quad
g_i =
\begin{bmatrix}
0.9 \\ 0.6 \\ 0.4
\end{bmatrix}
\]

The \textbf{motivational gradient} (Appendix~\ref{apx:ValMot}) is computed as:

\[
\bm{M}_i = g_i - x_i =
\begin{bmatrix}
0.5 \\ 0.3 \\ 0.2
\end{bmatrix}
\]

This vector \( \bm{M}_i \) encodes the direction in which the agent is internally inclined to act. Its components suggest the strongest drive is toward acquiring knowledge, followed by reputation and income.

In this framework, \textit{cognition} is the agent's awareness of their position and goal in the space, while \textit{desire} is expressed as the motivational gradient that links them. Together, they define purposeful motion in value space.


\vspace{0.5em}
\subsection{Concepts as Abstract Beings}

In my framework, concepts are not merely scalar propositions but structured entities I call \textit{abstract beings}. These are vectors residing within an agent’s value space.

An abstract being “exists” for an agent if it carries nonzero weight under that agent’s valuation. If no agent in a population retains a meaningful version of a being, I say the being has "died". This enables us to formally model the birth, evolution, and extinction of conceptual or ideological entities. See Appendix~\ref{apx:AbsBe} for the formalization.

\paragraph{Trans-Agent Identity of Abstract Beings.}
A natural question is when two structurally similar but quantitatively divergent representations $X_A \in \mathcal{V}_A$ and $X_B \in \mathcal{V}_B$ count as the \emph{same} abstract being, as opposed to distinct beings that merely resemble one another. The framework answers this geometrically rather than metaphysically: $X_A$ and $X_B$ are instantiations of the same abstract being to the extent that there exist interpretation maps $T_{A\to B}$ and $T_{B\to A}$ that satisfy the local coherence condition (Appendix~\ref{apx:LocalCoherence}) on a concept subspace $\mathcal{U}\subseteq\mathcal{V}_A$ containing $X_A$, and that transport $X_A$ and $X_B$ into approximate images of one another: $T_{A\to B}(X_A)\approx X_B$ and $T_{B\to A}(X_B)\approx X_A$ within the resulting round-trip distortion $\varepsilon$. This makes trans-agent identity \emph{graded} (governed by $\varepsilon$), \emph{asymmetric} (forward and backward distortion need not coincide), and grounded in the algebraic structure of the maps rather than in any external criterion of sameness. The lifecycle account that follows---birth, evolution, and death of abstract beings (Appendix~\ref{apx:abstractBeing})---inherits this graded identity: an abstract being persists across agents not by carrying a fixed essence, but by remaining recoverable under the interpretive geometry that connects them.

\subsubsection{Example: Alignment of Concepts}

In this framework, a concept is modeled as a fixed vector \( \mathbf{b} \) in value space — an abstract entity that "pulls" agents toward certain directions.

Suppose concept \( \mathbf{b} \) represents the idea: \emph{“Service to others brings fulfillment”}, encoded as:

\[
\mathbf{b} =
\begin{bmatrix}
0.6 \\ 0.8 \\ 0.0
\end{bmatrix}
\quad \text{(high in social and personal value, low in monetary gain)}
\]

An agent whose motivational gradient \( \bm{M}_i = g_i - x_i \) is:

\[
\bm{M}_i =
\begin{bmatrix}
0.5 \\ 0.5 \\ 0.1
\end{bmatrix}
\]

will find this concept partially aligned. The degree of alignment can be computed via the dot product:

\[
\mathbf{b} \cdot \bm{M}_i = 0.6 \cdot 0.5 + 0.8 \cdot 0.5 + 0 \cdot 0.1 = 0.7
\]

The higher the alignment, the more the concept resonates with the agent. Thus, concepts function as directional attractors in value space—abstract beings that influence and organize desire.

\vspace{0.5em}
\subsection{Interpretation Maps and Epistemic Compatibility}

Communication between agents is modeled via \emph{interpretation maps}:
linear transformations between value spaces. For agent $A$ to influence
agent $B$, the abstract being must survive this transformation in a way that
preserves both representational structure and motivational significance.

Successful influence requires three forms of consistency:
\begin{itemize}
  \item \textbf{Forward Consistency}: the transmitted being retains its
        representational structure (up to small distortion) under
        $T_{A\to B}$.
  \item \textbf{Backward Consistency}: $B$'s internal representation can be
        approximately mapped back to $A$'s original, enabling mutual
        intelligibility.
  \item \textbf{Valuation Consistency}: the concept retains or gains
        subjective value in the recipient’s space, preventing motivational
        collapse.
\end{itemize}

These conditions model influence not as a passive transfer of information,
but as a negotiated transformation of meaning across heterogeneous cognitive
geometries. See Appendix~\ref{apx:InMap} for the formalization.

\subsubsection{Admissibility of Interpretation Maps: The Shared-Basis Hypothesis}
\label{sec:Admissibility}

The three consistency conditions characterize \emph{successful} influence but leave open a more fundamental question: under what empirical regularities are interpretation maps approximately invertible on the concepts that matter? Stated as bare mathematics, $T_{A\to B}$ could be any linear map---including maps whose null space annihilates every concept either agent cares about. If the framework's predictions are to engage with observable communication, some structural account of when $T$ is well-behaved is needed.

The hypothesis adopted here is that interpretation maps are constrained by \emph{shared task bases}. Agents who participate in a common occupational, disciplinary, or social practice acquire, through repeated joint activity, the same basis vectors needed to perform the practice's work. Two engineers reasoning about load-bearing structures, two clinicians reasoning about patient management, or two negotiators reasoning about an agreement do not merely happen to agree---they share the dimensions along which their representations of the joint task are organized. On this shared task subspace $\mathcal{U} \subseteq \mathcal{V}_A$, the round-trip operator $T_{B\to A} \circ T_{A\to B}$ satisfies the local coherence condition (Appendix~\ref{apx:LocalCoherence}) for small $\varepsilon$: paired concepts survive cross-agent transmission with bounded distortion.

Null spaces become consequential precisely at the boundaries of shared bases---in cross-occupational, cross-disciplinary, or cross-cultural exchanges. An engineer's reasoning about algorithmic complexity, a philosopher's reasoning about transcendental idealism, or a clinician's reasoning about autoimmune etiology each invokes basis vectors that have no neighbors in the receiver's task subspace. The framework predicts that such concepts will be largely annihilated under cross-domain interpretation---not because the receiver is irrational or under-informed, but because the basis required to recover them is absent.

This view aligns with existing accounts of communicative grounding. Clark's notion of \emph{common ground}~\cite{clark1996using} captures the shared base that allows interlocutors to coordinate without re-explanation; the present framework gives that notion a geometric form, identifying common ground with the shared task subspace on which $T$ is near-invertible. An empirical illustration of this pattern, using sentence-embedding spaces to estimate $T$ between philosophical and engineering vocabularies, is provided in Section~\ref{sec:EmpiricalIllustration}. A full treatment of field-level basis-sharing dynamics---how shared bases form, fracture, and are renegotiated across communities---is the subject of a planned extension to epistemic flow across disciplines.

This shared-basis grounding also bears on a deeper question about the framework's commitment to linearity. A reader sympathetic to constructivist or pragmatist semantics might worry that imposing globally linear interpretation maps already presupposes the representational stability the framework purports to explain. The Admissibility hypothesis addresses this concern directly: bases are not treated as metaphysically pre-given but as acquired through repeated participation in shared practice. Linearity is then the local algebraic structure of this practice-acquired representational scaffolding rather than a prior commitment to fixed mental furniture. On this reading, the framework's geometric apparatus is compatible with constructivist and pragmatist positions, which similarly locate representational stability in practice rather than in any pre-given mental structure.

For simplicity, $\|\cdot\|$ denotes a vector norm used to approximate
valuation magnitudes; when necessary, $\|\cdot\|_i$ refers to the specific
norm induced by agent $i$'s valuation.

\begin{remark}
Throughout the examples and in the formal results of the appendix (Theorems~\ref{thm:MotiConv} and~\ref{thm:ConvHull}), the valuation function $\mathrm{Val}_i$ is taken to be a norm or a component sum. The theorems as stated depend on this simplifying choice; extending them to general valuation functions is a direction for future work.
\end{remark}

In the examples below, $T_A$ and $T_B$
map from a shared conceptual space into the value spaces of agents $A$ and
$B$, while $T_{i\to j}$ denotes maps between agents' value spaces.

\subsubsection{Example: Epistemic Compatibility}

Interpretation maps determine how abstract concepts are projected into an
agent's value space. Suppose the concept $C$—e.g., \emph{Freedom}—is
represented in a two-dimensional conceptual space as
\[
\mathbf{c} =
\begin{bmatrix}
1.0 \\ 0.5
\end{bmatrix}
\quad \text{(autonomy, opportunity)}.
\]

Agent $A$ interprets $C$ through the linear map
\[
T_A =
\begin{bmatrix}
1 & 0 \\
0 & 0.5 \\
0 & 0
\end{bmatrix}
\quad \text{(independence, achievement, profit)},
\]
yielding
\[
T_A(\mathbf{c}) =
\begin{bmatrix}
1.0 \\ 0.25 \\ 0
\end{bmatrix}.
\]

Agent $B$ interprets $C$ through
\[
T_B =
\begin{bmatrix}
0.5 & 1 \\
0 & 0 \\
0.2 & 0.1
\end{bmatrix}
\quad \text{(reads $C$ as independence with a small profit component; achievement is annihilated)},
\]
yielding
\[
T_B(\mathbf{c}) =
\begin{bmatrix}
1.0 \\ 0 \\ 0.25
\end{bmatrix}.
\]

Both agents regard $C$ as valuable, but they project it into \emph{partially
overlapping yet differently weighted subspaces}. Although a shared component
(profit) remains, the differing emphasis on autonomy or opportunity may
produce systematic misunderstanding. \emph{Epistemic compatibility} is thus
defined as the degree to which interpretation maps preserve nonzero
projections along similar directions.

\subsubsection{Example: Conditions for Successful Influence}

Consider the concept $B$: \emph{``Open-source education empowers society’’},
represented as
\[
\mathbf{b} =
\begin{bmatrix}
1.0 \\ 0.6
\end{bmatrix}
\quad \text{(freedom, collective benefit)}.
\]

\paragraph{Forward Consistency.}
Agent $A$ interprets the concept through
\[
T_A =
\begin{bmatrix}
1 & 0.5 \\
0 & 0.8 \\
0 & 0
\end{bmatrix}
\Rightarrow
T_A(\mathbf{b}) =
\begin{bmatrix}
1.3 \\ 0.48 \\ 0
\end{bmatrix}.
\]

Agent $B$ interprets it through
\[
T_B =
\begin{bmatrix}
0.6 & 0.9 \\
0.1 & 0.1 \\
0.2 & 0
\end{bmatrix}
\Rightarrow
T_B(\mathbf{b}) =
\begin{bmatrix}
1.14 \\ 0.16 \\ 0.2
\end{bmatrix}.
\]

Both images preserve similar \emph{rank-order emphasis} across dimensions: a
dominant first component, a moderate second, and a small third. This
approximate preservation of structural relations constitutes forward
consistency.

\paragraph{Backward Consistency.}
Agent $A$ attempts to recover its own representation from $B$'s
interpretation using a map $T_{B\to A}$:
\[
T_{B\to A}\big(T_B(\mathbf{b})\big) \approx T_A(\mathbf{b}).
\]
Here $T_{B\to A}$ denotes any reconstruction map from $B$'s value space
back to $A$'s; exact invertibility is not required. Successful
approximation indicates that $A$ can recover a meaningful version of
its own representation from $B$'s internal encoding, ensuring mutual
intelligibility.

\paragraph{Valuation Consistency.}
Each agent evaluates the concept within their respective value spaces:
\[
\mathrm{Val}_A(\mathbf{b}) = \|T_A(\mathbf{b})\| \approx 1.38, \qquad
\mathrm{Val}_B(\mathbf{b}) = \|T_B(\mathbf{b})\| \approx 1.17.
\]
Although the images differ structurally, their magnitudes remain close, and
the directional similarity (e.g., cosine similarity near~1) ensures that the
motivational content of the concept is largely preserved.

\medskip

Together, these three consistencies ensure that concepts are not merely
transmitted but \emph{intelligible, reconstructible, and motivationally
resonant} across agents. This provides a geometric criterion for genuine
influence and shared understanding.

\vspace{0.5em}
\subsection{The Geometry of Mutual Understanding}

If two agents can consistently map concepts back and forth, they achieve what I call \textit{local coherence}. This enables stable communication, even if their internal representations differ. This property is made precise in Appendix~\ref{apx:LocalCoherence}, where I show that if each agent can approximately reconstruct the other’s representation, then repeated communication preserves concept stability (Theorem~\ref{thm:LocalCoherence}).

Philosophically, this expresses a form of representational intersubjectivity—agents may differ cognitively but still maintain stable mutual understanding. 

\subsubsection{Example: The Geometry of Mutual Understanding}

Mutual understanding arises when two agents’ interpretation maps project shared concepts into partially aligned directions in value space.

Let abstract object \( X \) represent a goal like \emph{“Starting a nonprofit tech school”}, encoded in conceptual space as:

\[
\mathbf{x} =
\begin{bmatrix}
0.7 \\ 0.8
\end{bmatrix}
\quad \text{(education, social impact)}
\]

Agent A’s interpretation map \( T_A \) projects into their value space:

\[
T_A =
\begin{bmatrix}
1 & 0.5 \\
0 & 0.6 \\
0 & 0
\end{bmatrix}
\quad \text{(values: autonomy, recognition, wealth)}
\]

\[
T_A(\mathbf{x}) =
\begin{bmatrix}
1.1 \\ 0.48 \\ 0
\end{bmatrix}
\]

Agent B interprets the same \( X \) via:

\[
T_B =
\begin{bmatrix}
0.4 & 1 \\
0.2 & 0 \\
0.1 & 0.1
\end{bmatrix}
\]

\[
T_B(\mathbf{x}) =
\begin{bmatrix}
1.08 \\ 0.14 \\ 0.15
\end{bmatrix}
\]

Though their internal bases differ, both agents perceive strong value in \( X \), with significant projection along socially meaningful dimensions. The dot product of their interpreted vectors shows high directional similarity:

\[
\cos \theta = \frac{T_A(\mathbf{x}) \cdot T_B(\mathbf{x})}{\|T_A(\mathbf{x})\|\|T_B(\mathbf{x})\|} \approx 0.96
\]

This near-alignment indicates high epistemic compatibility and \textit{mutual resonance}. Thus, mutual understanding can be modeled as \textit{angular proximity} between transformed conceptual vectors.

\vspace{0.5em}
\subsection{Cognitive Blindness and Null Spaces}

Each interpretation map has a \textit{null space}—a subspace of concepts that the recipient agent cannot perceive or evaluate. This provides a powerful model of \textit{cognitive blindness}: the idea that some concepts are structurally invisible to others, not due to ignorance or bias, but because they lie outside the range of interpretive capacity.

This insight explains why even well-intended communication can fail entirely, and it provides a formal account of ideological opacity and interpretive misfire. See Appendix~\ref{apx:NullCog} for the formalization.

\subsubsection{Example: Perceptual Omission}

Consider a goal \( X \): creating non-commercial art, represented in Person A’s space as:

\[
\mathbf{x}_A =
\begin{bmatrix}
1.0 \\ 0.5 \\ 0.0
\end{bmatrix}
\]

with dimensions: creativity (\( b_1 \)), social utility (\( b_2 \)), and financial reward (\( b_3 \)).

Person B’s perspective transformation is:

\[
T_{A \rightarrow B} =
\begin{bmatrix}
0 & 0 & 0 \\
0 & 1 & 0 \\
0 & 0 & 1
\end{bmatrix}
\]

This matrix removes sensitivity to creative value (\( b_1 \)). Applying the transformation:

\[
\mathbf{x}_B = T_{A \rightarrow B} \cdot \mathbf{x}_A =
\begin{bmatrix}
0 \\ 0.5 \\ 0
\end{bmatrix}
\]

Now consider a purely creative act:

\[
\mathbf{z}_A = 
\begin{bmatrix}
1 \\ 0 \\ 0
\end{bmatrix}
\quad \Rightarrow \quad
T_{A \rightarrow B} \mathbf{z}_A = \mathbf{0}
\]

To Person B, this project has \emph{zero perceived value}. Although meaningful in A’s space, it lies in the null space of B’s perspective—hence, cognitively invisible.

This formalizes how deep value misalignment can lead not just to disagreement, but to total perceptual omission.

\vspace{0.5em}
\subsection{Networks of Influence}

I generalize my model to multi-agent networks, where influence occurs over time and along paths in a graph. If an abstract being can travel along a connected sequence of interpretive links—and survives the cognitive filters at each stage—it eventually reaches distant agents.

This contrasts with diffusion models that focus solely on probabilistic spread. Here, influence is constrained not just by social topology, but by the composition of interpretive maps and null spaces along the path. More details can be found in Appendix~\ref{apx:NetInfl}.

\subsubsection{Example: Propagation in a Network of Interpretive Maps}

Let an abstract concept \( \mathbf{b} \in \mathbb{R}^2 \), such as \emph{“Decentralized knowledge is more resilient”}, originate with agent \( A_1 \), represented as:

\[
\mathbf{b} =
\begin{bmatrix}
1.0 \\ 0.7
\end{bmatrix}
\quad \text{(freedom, epistemic trust)}
\]

This concept travels through a directed network of agents \( A_1 \rightarrow A_2 \rightarrow A_3 \), each with a linear interpretation map \( T_i \). To make it simple, let all agents can influence others with probability 1. The concept propagates only if it survives all transformations:

\begin{align*}
\mathbf{b}_2 &= T_2 \cdot T_1(\mathbf{b}) \\
\mathbf{b}_3 &= T_3 \cdot \mathbf{b}_2
\end{align*}

Suppose:
\[
T_1 =
\begin{bmatrix}
1 & 0 \\
0.2 & 0.8
\end{bmatrix}, \quad
T_2 =
\begin{bmatrix}
1 & 0 \\
0 & 0.5
\end{bmatrix}, \quad
T_3 =
\begin{bmatrix}
0 & 1 \\
0 & 0
\end{bmatrix}
\]

Then:
\[
T_1(\mathbf{b}) =
\begin{bmatrix}
1.0 \\
0.2(1.0) + 0.8(0.7) = 0.2 + 0.56 = 0.76
\end{bmatrix}
\]

\[
\mathbf{b}_2 = T_2 \cdot
\begin{bmatrix}
1.0 \\ 0.76
\end{bmatrix}
=
\begin{bmatrix}
1.0 \\ 0.38
\end{bmatrix}, \quad
\mathbf{b}_3 = T_3 \cdot
\begin{bmatrix}
1.0 \\ 0.38
\end{bmatrix}
=
\begin{bmatrix}
0.38 \\ 0
\end{bmatrix}
\]

Here, the concept's second component is lost at \( A_3 \) due to a cognitive filter (null row), meaning it survives only partially.

This demonstrates that influence is not purely probabilistic—as in classical diffusion models—but geometrically constrained by the structure and composition of interpretive maps. Only concepts aligned with the non-null subspace of each agent's transformation can propagate successfully through the network.

\vspace{0.5em}
\subsection{No-Null-Space Leadership and Geometric Authority}

I introduce the concept of \textit{No-null-space leadership}: a leader can influence an agent only if their concept survives all the interpretive transformations along the path from leader to follower.

This leads to a formal criterion:
\begin{quote}
A leader $L$ leads agent $i$ with respect to concept $X$ if and only if the \textit{composite interpretation map} (see Appendix~\ref{apx:CompNullFilters}) from $L$ to $i$ does not annihilate $X$.
\end{quote}

Importantly, leadership is no longer about hierarchy or charisma—it is a structural property of cognitive compatibility. More details for No-Null-Space Leadership Condition are at Appendix~\ref{apx:noNullLeaderCond}.

\subsubsection{Example: Geometric Authority of Concept Spreading}

Let a leader \( L \) attempt to transmit a concept \( \mathbf{x} \in \mathbb{R}^3 \), such as:

\[
\mathbf{x} =
\begin{bmatrix}
0.8 \\ 0.6 \\ 0.4
\end{bmatrix}
\quad \text{(autonomy, fairness, collaboration)}
\]

The concept must pass through agents \( A_1 \rightarrow A_2 \rightarrow A_3 \), each applying their interpretive map \( T_i \). The leader’s influence succeeds only if the composite transformation does not send \( \mathbf{x} \) to the zero vector:

\[
T_{\text{net}} = T_3 \cdot T_2 \cdot T_1
\quad \text{such that} \quad T_{\text{net}}(\mathbf{x}) \neq \mathbf{0}
\]

Suppose the intermediate transformations are:

\[
T_1 =
\begin{bmatrix}
1 & 0 & 0 \\
0 & 1 & 0 \\
0 & 0 & 1
\end{bmatrix}, \quad
T_2 =
\begin{bmatrix}
1 & 0 & 0 \\
0 & 1 & 0 \\
0 & 0 & 0
\end{bmatrix}, \quad
T_3 =
\begin{bmatrix}
0 & 1 & 0 \\
0 & 0 & 0 \\
0 & 0 & 0
\end{bmatrix}
\]

Apply stepwise:

\[
T_2 \cdot T_1(\mathbf{x}) = 
\begin{bmatrix}
0.8 \\ 0.6 \\ 0
\end{bmatrix}, \quad
T_3 \cdot
\begin{bmatrix}
0.8 \\ 0.6 \\ 0
\end{bmatrix} =
\begin{bmatrix}
0.6 \\ 0 \\ 0
\end{bmatrix}
\]

Since the final result is non-zero, the concept survives the path, and I say:

\[
L \leadsto A_3 \quad \text{with respect to } \mathbf{x}
\]

Now suppose \( \mathbf{x}' = [0, 0, 1]^\top \). This vector is annihilated at \( T_2 \), so:

\[
T_{\text{net}}(\mathbf{x}') = \mathbf{0}
\quad \Rightarrow \quad
L \nrightarrow A_3 \text{ (with respect to } \mathbf{x}' \text{)}
\]

\textbf{Conclusion:} Leadership is no longer defined by position or authority, but by the structural survivability of concepts through interpretive filters. A leader is one whose concepts remain visible throughout the network: their authority is \emph{geometric}, not symbolic.

\vspace{0.5em}
\subsection{Leadership Emergence and Its Boundaries}

Leadership is operationalized as the successful transmission of a concept and its valuation. An agent is led when it adopts a concept from another, not only structurally but motivationally. Using the no-null-space leadership condition, I determine exactly which agents a leader can influence.

This highlights the limitations of leadership in epistemically diverse populations. Some agents are unreachable—not by social distance, but by incompatibility in value representation.
The rigorous details of Leader Emergence can be found at Appendix~\ref{apx:leaderEmerg}.

\subsubsection{Example: Leadership as Transmission of Structure and Value}

Let a leader \( L \) hold a concept \( \mathbf{x} \in \mathbb{R}^3 \), such as:

\[
\mathbf{x} =
\begin{bmatrix}
0.9 \\ 0.7 \\ 0.3
\end{bmatrix}
\quad \text{(justice, autonomy, empathy)}
\]

Suppose this concept must propagate to two agents: \( A_1 \) and \( A_2 \). Each agent applies a different interpretive map:

\[
T_1 =
\begin{bmatrix}
1 & 0 & 0 \\
0 & 0.8 & 0 \\
0 & 0 & 0.6
\end{bmatrix}, \quad
T_2 =
\begin{bmatrix}
0 & 0 & 0 \\
0 & 0 & 0 \\
0 & 0 & 1
\end{bmatrix}
\]

\textbf{Agent \( A_1 \)} receives:

\[
T_1(\mathbf{x}) =
\begin{bmatrix}
0.9 \\ 0.56 \\ 0.18
\end{bmatrix}, \quad \mathrm{Val}_{A_1}(\mathbf{x}) = \|T_1(\mathbf{x})\| \approx 1.08.
\]

The concept retains structure and motivation. \( A_1 \) is successfully led by \( L \).

\textbf{Agent \( A_2 \)} receives:

\[
T_2(\mathbf{x}) =
\begin{bmatrix}
0 \\ 0 \\ 0.3
\end{bmatrix}, \quad \mathrm{Val}_{A_2}(\mathbf{x}) = \|T_2(\mathbf{x})\| \approx 0.3.
\]

Most of the concept is annihilated due to \( T_2 \)’s narrow subspace. \( A_2 \) does not receive the core message nor enough motivation. Thus, \( L \) fails to lead \( A_2 \) with respect to \( \mathbf{x} \).

\textbf{Conclusion:} Leadership requires that the concept not fall into the null space of the follower’s interpretation—and that its projected value remains motivational. Some agents, like \( A_2 \), are not leadable by structural incompatibility, not distance. This reveals the geometric limits of leadership in epistemically diverse populations.

\vspace{0.5em}
\subsection{Motivational Alignment}

When a follower adopts a concept, it shifts their goal state, altering their motivational gradient. Over time, this causes directional alignment within a network, as agents begin to act in ways consistent with the leader’s representation of reality.

This offers a formal account of how concept adoption leads to action, and how leadership can realign goals through cognitive transmission.

\subsubsection{Example: Motivational Alignment through Concept Adoption}

Motivational alignment occurs when a transmitted concept alters an agent’s internal goal vector, changing the direction of their effort. This directional shift reflects not just agreement, but realigned behavior.

Let leader \( L \) hold concept \( \mathbf{x} \in \mathbb{R}^3 \), representing:

\[
\mathbf{x} =
\begin{bmatrix}
0.9 \\ 0.7 \\ 0.4
\end{bmatrix}
\quad \text{(service, innovation, reputation)}
\]

Let follower \( A \)’s initial state and goal be:

\[
x_A =
\begin{bmatrix}
0.4 \\ 0.2 \\ 0.5
\end{bmatrix}, \quad
g_A^{(0)} =
\begin{bmatrix}
0.5 \\ 0.3 \\ 0.6
\end{bmatrix}
\]

Then the initial motivational gradient is:

\[
\bm{M}_A^{(0)} = g_A^{(0)} - x_A =
\begin{bmatrix}
0.1 \\ 0.1 \\ 0.1
\end{bmatrix}
\]

After adopting \( \mathbf{x} \), given $\alpha$ is an influence factor that $L$ can influence $A$'s goal, follower \( A \) shifts their goal to:

\[
g_A^{(1)} = \alpha \cdot \mathbf{x} + (1 - \alpha) \cdot g_A^{(0)}, \quad \alpha = 0.6
\]

\[
\Rightarrow g_A^{(1)} = 
0.6 \cdot 
\begin{bmatrix}
0.9 \\ 0.7 \\ 0.4
\end{bmatrix}
+ 0.4 \cdot
\begin{bmatrix}
0.5 \\ 0.3 \\ 0.6
\end{bmatrix}
=
\begin{bmatrix}
0.74 \\ 0.54 \\ 0.48
\end{bmatrix}
\]

New motivational gradient:

\[
\bm{M}_A^{(1)} = g_A^{(1)} - x_A =
\begin{bmatrix}
0.34 \\ 0.34 \\ -0.02
\end{bmatrix}
\]

This shows a shift in direction: the agent now moves toward service and innovation, slightly away from reputation. The angle between \( \bm{M}_A^{(1)} \) and the leader’s concept vector is small:

\[
\cos \theta = \frac{\bm{M}_A^{(1)} \cdot \mathbf{x}}{\|\bm{M}_A^{(1)}\| \cdot \|\mathbf{x}\|} \approx 0.92
\]

\textbf{Conclusion:} The agent's adopted concept reshapes their goals and reorients their behavior. Leadership thus operates not by command, but by transmitting directional attractors that gradually realign agents' motivational dynamics.

\vspace{0.5em}
\subsection{Coordination as Distributed Alignment}

Coordination is a collective behavior that occurs in a group. A leader can be consider as an initiator of this collective behavior that everyone follows~\cite{10.1145/3201406}. Coordination occurs when a concept \( \mathbf{x} \) originating from a leader \( L \) is successfully and consistently interpreted across all agents in a group—both structurally and motivationally.

Let \( \mathbf{x}_L \in \mathbb{R}^n \) represent the leader’s concept vector. Each agent \( A_i \) interprets it via their transformation \( T_i \), producing:

\[
\mathbf{x}_i = T_i(\mathbf{x}_L)
\]

I define two conditions for coordination:

\begin{enumerate}
  \item \textbf{Structural Consistency:} Each agent’s interpreted concept must approximate the leader’s original, up to a small tolerance \( \epsilon \):

  \[
  \| \mathbf{x}_i - \mathbf{x}_L \| < \epsilon, \quad \forall i
  \]

  \item \textbf{Valuation Consistency:} The subjective value of the concept (measured by vector norm, valuation function $\mathrm{Val}(\cdot)$ or motivational projection, but I use the norm for simplicity) must be similar to the leader’s valuation, up to a small bound \( \delta \):

  \[
  | \|\mathbf{x}_i\| - \|\mathbf{x}_L\| | < \delta, \quad \forall i
  \]
\end{enumerate}

\subsubsection{Example: Coordination}

Let the leader's concept be:

\[
\mathbf{x}_L =
\begin{bmatrix}
0.9 \\ 0.6 \\ 0.3
\end{bmatrix}, \quad \|\mathbf{x}_L\| \approx 1.12
\]

Suppose three agents apply:

\[
T_1 = I, \quad
T_2 =
\begin{bmatrix}
0.95 & 0 & 0 \\
0 & 1.05 & 0 \\
0 & 0 & 1.0
\end{bmatrix}, \quad
T_3 =
\begin{bmatrix}
1.0 & 0.05 & 0 \\
-0.02 & 0.98 & 0 \\
0 & 0 & 1.01
\end{bmatrix}
\]

Each produces \( \mathbf{x}_i \) such that:

\[
\|\mathbf{x}_i - \mathbf{x}_L\| < \epsilon = 0.1, \quad
|\|\mathbf{x}_i\| - \|\mathbf{x}_L\|| < \delta = 0.05
\]

Thus, all agents interpret the concept similarly and value it nearly equally. Coordination is achieved: the group shares a structurally consistent and motivationally aligned representation of the leader’s concept. In this geometric framework, coordination is a condition of distributed alignment. It emerges not from consensus or conformity, but from bounded transformation and shared valuation. The rigorous details of motivational convergence, which is a form of coordination can be found at Appendix~\ref{apx:MotiConv}.

\vspace{0.5em}
\subsection{Concept Alignment}

Concept alignment occurs when an agent not only accepts a concept structurally, but also adjusts its internal value system so that the concept is valued similarly to how the leader values it. This deeper form of alignment requires not just cognitive agreement, but motivational resonance.

Let \( \mathbf{x} \in \mathbb{R}^n \) represent an abstract being—a concept or principle—that leader \( L \) transmits. Let \( T_A \) be agent A’s original interpretive matrix, and \( \|T_A(\mathbf{x})\| \) the subjective value A assigns to the belief.

If agent A adopts the concept but initially perceives its value differently than the leader:

\[
\|\mathbf{x}_L\| = 1.2, \quad \|T_A(\mathbf{x})\| = 0.6
\]

then A must adjust its internal basis to match the leader’s valuation. This adjustment is achieved by applying a transformation matrix \( M_{L \rightarrow A} \), called the \textbf{Persuasion Matrix}, yielding:

\[
T_A' = M_{L \rightarrow A} \cdot T_A
\quad \text{such that} \quad
\|T_A'(\mathbf{x})\| \approx \|\mathbf{x}_L\|
\]

The Persuasion Matrix encodes the leader’s cognitive influence: it represents A’s internal restructuring of priorities in response to $L$’s effort to convince, teach, or re-frame. The result is not just structural understanding, but a shift in value alignment.

\paragraph{Where do Persuasion Matrices come from?}
A natural question is how persuasion matrices are generated in practice. Three mechanisms are plausible. First, a persuasion matrix can be \emph{learned from shared interaction data}. If agents $L$ and $A$ communicate over a shared corpus of concept pairs---cases where both assign meaning to the same entity---then the matrix $M_{L \to A}$ can be estimated by solving a least-squares alignment problem analogous to cross-lingual embedding alignment~\cite{conneau2017word}. Formally, given paired representations $\{(T_A(x_k), x_k^L)\}_{k=1}^K$, one solves
\[
M_{L \to A} = \arg\min_{M} \sum_{k=1}^{K} \|M \cdot T_A(x_k) - x_k^L\|^2,
\]
which reduces to the classical Procrustes problem~\cite{schonemann1966generalized} when $M$ is constrained to be orthogonal. Second, the matrix can emerge through \emph{iterative discourse}: repeated exchanges in which $A$ adjusts their evaluative weights in response to feedback from $L$, converging on an alignment that minimizes valuation discrepancy. Third, the matrix can be \emph{inferred from behavioral observation}: an external observer (or the agents themselves) can regress observed changes in $A$'s behavior onto the structure of $L$'s transmitted concepts, yielding an empirical estimate of the transformation that $L$'s influence induces. In each case, the persuasion matrix is not stipulated but derived from interaction, learning, or observation.

\subsubsection{Example: Basis Adjustment through Persuasion}

Let the concept be:

\[
\mathbf{x} =
\begin{bmatrix}
0.8 \\ 0.6 \\ 0.2
\end{bmatrix}
\quad \Rightarrow \quad
\|\mathbf{x}_L\| \approx 1.03
\]

A’s original matrix is:

\[
T_A =
\begin{bmatrix}
0.5 & 0 & 0 \\
0 & 0.5 & 0 \\
0 & 0 & 0.5
\end{bmatrix}
\quad \Rightarrow \quad
\|T_A(\mathbf{x})\| \approx 0.57
\]

To achieve alignment, A applies the Persuasion Matrix:

\[
M_{L \rightarrow A} =
\begin{bmatrix}
2 & 0 & 0 \\
0 & 2 & 0 \\
0 & 0 & 1.6
\end{bmatrix}
\quad \Rightarrow \quad
T_A' = M_{L \rightarrow A} \cdot T_A
\]

\[
T_A' =
\begin{bmatrix}
1.0 & 0 & 0 \\
0 & 1.0 & 0 \\
0 & 0 & 0.8
\end{bmatrix}
\quad \Rightarrow \quad
\|T_A'(\mathbf{x})\| \approx 1.02 \approx \|\mathbf{x}_L\|
\]

Now A values the concept nearly equally to the leader. This constitutes concept alignment: A has reframed its internal space to match the leader’s motivational landscape.

\textbf{Conclusion:} The Persuasion Matrix formalizes how leadership can reshape an agent’s evaluative geometry. Concept alignment occurs not merely through agreement, but through internal transformation of value space—making persuasion a geometric operation on cognition.

\vspace{0.5em}
\subsection{Convex-Hull Leadership and Cultural Innovation}

I examine how leadership can push followers beyond their current range of values. If a leader’s valuation lies outside the convex hull of the group’s existing value structure, adopting their concept expands the group’s conceptual boundaries. This is consistent with the work~\cite{amornbunchornvej2020framework}, which states that a leader is the one who steps outside a convex-hull of the group before others in a local-agreement-system. 

This explains how leaders introduce cultural, ethical, or strategic innovations not emergent from internal consensus, but from structural difference. See Appendix~\ref{apx:convLead} for more rigorous details of Convex-Hull Leadership.

\subsubsection{Example: Innovation beyond Existing Concepts}

Let agents \( A_1, A_2, A_3 \) have respective value directions (interpreted basis vectors):

\[
\mathbf{v}_1 =
\begin{bmatrix}
1 \\ 0 \\ 0
\end{bmatrix}, \quad
\mathbf{v}_2 =
\begin{bmatrix}
0 \\ 1 \\ 0
\end{bmatrix}, \quad
\mathbf{v}_3 =
\begin{bmatrix}
0 \\ 0.5 \\ 1
\end{bmatrix}
\]

The convex hull of these vectors contains all concepts expressible as weighted combinations of them:

\[
\mathbf{x}_{\text{group}} = \alpha_1 \mathbf{v}_1 + \alpha_2 \mathbf{v}_2 + \alpha_3 \mathbf{v}_3, \quad \alpha_i \geq 0, \quad \sum \alpha_i = 1
\]

Now suppose leader \( L \) holds a concept:

\[
\mathbf{x}_L =
\begin{bmatrix}
0.6 \\ 0.6 \\ 1
\end{bmatrix}
\]

This vector cannot be expressed as a convex combination of \( \mathbf{v}_1, \mathbf{v}_2, \mathbf{v}_3 \). It lies outside the group's current value boundary.

\textbf{Innovation through Adoption:} If agents adopt \( \mathbf{x}_L \), they expand their shared span. The group’s new convex hull includes regions previously unreachable—altering what can be collectively imagined or pursued.

This explains how visionary leaders introduce cultural or ethical ideas that do not arise from consensus (interpolation), but from \textit{structural difference}—they shift the group by \emph{enlarging the shape of the space}.

\textbf{Conclusion:} Convex-hull leadership reframes innovation as geometric disruption. Influence, in this case, is not alignment within existing norms— but expansion of the possible.

\subsubsection{Beyond the Convex Hull: Basis Expansion}
The preceding analysis considers innovation within a fixed set of evaluative dimensions. A deeper form of innovation occurs when a leader introduces an entirely \emph{new dimension} into the group's representational basis— not merely a new point outside the existing hull, but a new axis along which evaluation becomes possible. This corresponds to expanding the value space itself: if the group's current basis is $\{b_1, \dots, b_d\}$, the leader introduces $b_{d+1}$, enlarging the space from $\mathbb{R}^d$ to $\mathbb{R}^{d+1}$.

Such basis expansion is warranted when the group's existing dimensions systematically fail to represent some aspect of experience. A principled criterion is that the new dimension should reduce representational error by more than its added complexity costs --- a tradeoff naturally formalized by a Minimum Description Length (MDL) criterion.

To make this precise, define the total description cost of a basis $\mathcal{B}_d = \{b_1, \dots, b_d\}$ used to describe $n$ observations as
\[
L(\mathcal{B}_d) = L_{\text{residual}}(\mathcal{B}_d) + L_{\text{basis}}(\mathcal{B}_d),
\]
where $L_{\text{residual}}$ encodes the unexplained variance (the squared residuals of the observations orthogonal to the span of $\mathcal{B}_d$) and $L_{\text{basis}}$ encodes the basis itself (the number of real coordinates needed to describe both the basis vectors and the observations at a chosen precision). A candidate dimension $b_{d+1}$ is admitted if and only if $L(\mathcal{B}_d \cup \{b_{d+1}\}) < L(\mathcal{B}_d)$.

This distinguishes genuine conceptual innovation from mere elaboration: a new dimension is justified only when the existing geometry is structurally inadequate, not simply when a novel direction is proposed. Leaders who introduce such dimensions do not merely shift the group's position in value space --- they reshape the space itself, enabling forms of evaluation, motivation, and coordination that were previously inexpressible.

\subsubsection{Example: MDL Basis Expansion}
Suppose the group's current basis is $\{b_{\text{utility}}, b_{\text{rights}}\}$ ($d = 2$) and a candidate $b_{\text{care}}$ is proposed by an innovator to capture the relational dimension of moral concern. The group has $n = 5$ moral evaluations whose underlying three-component representations are
\[
v_1 = (0.3, 0.4, 0.9), \; v_2 = (0.5, 0.7, 0.6), \; v_3 = (0.6, 0.3, 0.4), \; v_4 = (0.2, 0.2, 0.8), \; v_5 = (0.4, 0.5, 0.5).
\]
Under the 2D basis, the care-component is orthogonal to the span and contributes the entire residual: $\sum_k \|r_k\|^2 = 0.81 + 0.36 + 0.16 + 0.64 + 0.25 = 2.22$. Under the 3D basis, the residual vanishes.

Using coordinate precision $q = 0.1$ (each real number costs $\ln(1/q) \approx 2.3$ nats, where a nat is the natural-logarithm unit of information) and Gaussian codelength for residuals with variance $\sigma^2 = q^2$, the two cost terms take the form
\[
L_{\text{residual}}(\mathcal{B}_d) = \frac{1}{2 q^2}\sum_k \|r_k\|^2, \qquad
L_{\text{basis}}(\mathcal{B}_d) = d(d + n)\ln(1/q),
\]
the latter accounting for $d$ basis vectors at $d$ coordinates each plus $n$ observation coordinates each, all at precision $q$. Substituting:
\begin{align*}
L(\mathcal{B}_2) &= \tfrac{2.22}{2 q^2} + 2(2+5)\ln(1/q) \;\approx\; 111 + 32.2 \;=\; 143.2 \text{ nats}, \\
L(\mathcal{B}_3) &= 0 + 3(3+5)\ln(1/q) \;\approx\; 0 + 55.2 \;=\; 55.2 \text{ nats}.
\end{align*}
The expansion saves about $88$ nats; the criterion admits $b_{\text{care}}$.

\vspace{0.5em}
\subsection{Birth, Evolution, and Death of Abstract Beings}

Finally, I describe the life cycle of abstract beings. They are born when instantiated by an agent, evolve as they are reinterpreted across agents, and die when they are no longer represented by any cognitive system.

This framework models conceptual systems, ideologies, and identities as structured entities with their own trajectories, shaped by the value geometries of the populations they inhabit. The details of Abstract beings formalization are at Appendix~\ref{apx:abstractBeing}.

\subsubsection{Example: Life Cycle of an Abstract Being}

Abstract beings—such as concepts, ideologies, or identities—are modeled as structured vectors that live, evolve, and disappear through cognitive representation. Their lifespan depends on their instantiation within agents’ value spaces.

\textbf{Birth:} An abstract being \( \mathbf{x} \in \mathbb{R}^3 \) is born when an agent instantiates it:

\[
\mathbf{x} =
\begin{bmatrix}
0.9 \\ 0.4 \\ 0.1
\end{bmatrix}
\quad \text{(justice, humility, aesthetic depth)}
\]

Suppose agent \( A_1 \) adopts \( \mathbf{x} \), assigning it high value:

\[
\|T_1(\mathbf{x})\| = 1.1
\]

\textbf{Evolution:} The concept spreads through a network of agents \( A_1 \rightarrow A_2 \rightarrow A_3 \), undergoing reinterpretation via transformation matrices:

\[
\mathbf{x}_2 = T_2 \cdot T_1(\mathbf{x}), \quad
\mathbf{x}_3 = T_3 \cdot \mathbf{x}_2
\]

Over time, \( \mathbf{x} \) mutates in structure and emphasis (e.g., increasing focus on recognition over humility), forming variants that retain partial identity.

\textbf{Death:} If no agents continue to represent \( \mathbf{x} \) with nonzero projection—i.e., \( T_i(\mathbf{x}) = \mathbf{0} \) for all \( i \)—then \( \mathbf{x} \) has no cognitive instantiation. It vanishes from the collective mind.

\textbf{Conclusion:} Abstract beings possess trajectories across populations: they are born via activation, evolve via reinterpretation, and die through cognitive neglect. Their lifespans are not governed by truth or utility alone, but by alignment with the value geometries of the systems they inhabit.

\subsubsection{Example: Constructing Value Spaces from Linguistic Embeddings}

To sketch one possible implementation, consider agents whose concepts are represented by natural language statements. Each agent $A$ maintains a value space $\mathcal{V}_A \subset \mathbb{R}^d$, where each concept $b_i$ is embedded into $\vec{v}_i^A \in \mathcal{V}_A$ using a large pre-trained language model (e.g., Sentence-BERT or GPT-derived embeddings).

Suppose a concept $b_i^A = \text{``Democracy requires transparency''}$ is represented as $\vec{v}_i^A \in \mathcal{V}_A$. Another agent $B$ holds a similar concept $b_j^B = \text{``Open data is essential for democratic governance''}$, embedded as $\vec{v}_j^B \in \mathcal{V}_B$. A structural alignment map $T_{A \rightarrow B}: \mathcal{V}_A \to \mathcal{V}_B$ could be learned by minimizing interpretive loss over a shared corpus $C_{AB}$, such that:

\[
T_{A \rightarrow B} = \arg\min_{T} \sum_{(b_i, b_j) \in C_{AB}} \left\| T(\vec{v}_i^A) - \vec{v}_j^B \right\|^2
\]

Here, $C_{AB}$ could consist of annotated concept pairs deemed semantically or motivationally aligned by human raters or alignment models.

Once $T_{A \rightarrow B}$ is estimated, concept survivability can be evaluated geometrically: a concept $\vec{v}_k^A$ is filtered out (``dies'') if it maps to the null space or outside of the receptive region of $\mathcal{V}_B$, i.e., if $\|T_{A \rightarrow B}(\vec{v}_k^A)\|$ is below a semantic or normative activation threshold.

\paragraph{Motivational Subspaces.} To identify which aspects of the embedding space correspond to motivation (e.g., fairness, authority, security), I can apply probing classifiers or concept activation vectors (CAVs) \cite{koh2020concept} trained on labeled motivational dimensions. These yield directional vectors $\vec{m}_k$ defining axes of motivational salience within $\mathcal{V}_A$, making it possible to analyze or visualize persuasion gradients.

\paragraph{Interpretation via Behavior.} Alternatively, behavioral data—such as policy votes, social media endorsements, or cooperative decisions—can be used to regress observed influence onto components of the concept vector, allowing empirical estimation of which components survive interpretive transfer across agents.

While this implementation is schematic, it suggests the feasibility of grounding value-space dynamics in observable linguistic or behavioral data using modern representation learning methods.
\section{Applications: Explanation of Real-World Scenarios}

\subsection{A Cognitive-Geometric Model of Social Identity Leadership}

Humans are intrinsically social organisms who continuously exchange affective and behavioral cues, with mimicry occurring even in basic responses such as laughter or yawning \cite{hatfield1993emotional}. Because group membership historically conferred survival advantages—and because belongingness constitutes a fundamental human motivation \cite{baumeister1995need}—social exclusion elicits strong regulatory responses. Individuals who experience rejection tend to exhibit heightened conformity and behavioral attunement to the group as a means of restoring threatened belonging needs~\cite{williams2000cyberostracism, williams2007ostracism}. These well-documented processes underscore a general human tendency toward assimilation, norm adherence, and sensitivity to social evaluation.

Such tendencies form the psychological basis for the emergence and maintenance of leadership. According to Social Identity Theory of Leadership \cite{hogg2001social}, leaders gain influence by (1) positioning themselves as highly prototypical of the group’s identity, a principle supported by extensive organizational research \cite{vanknippenberg2003social}. They further maintain authority by (2) reinforcing normative boundaries through the reward of conformity and the sanction of deviance, consistent with the well-established ``black sheep effect,'' in which in-group deviants are evaluated especially harshly \cite{marques1994blacksheep,jetten2004distinctiveness}. Finally, leaders consolidate cohesion by (3) accentuating contrasts with relevant out-groups—an identity-based strategy central to influential models of leadership and collective mobilization \cite{haslam2020new,escalas2005self}. Because individuals are predisposed to avoid deviance and align with perceived group norms, leaders who successfully define and regulate group identity can effectively stabilize and extend their authority.

Here I develop a cognitive-geometric framework that formalizes these mechanisms within a shared value-space representation of concepts and interpersonal perception. I will formalize these three points using the Cognitive-Geometric Model in the next section. 

\subsubsection{Value Spaces and Perspective Transformations}

Each agent $i$ possesses a value space $\mathcal{V}_i$ and a valuation function 
\[
\mathrm{Val}_i : \mathcal{V}_i \to \mathbb{R},
\]
encoding the subjective importance of any concept vector.  
Agent $i$ expresses a concept or stance $x_i \in \mathcal{V}_i$.  
When agent $j$ observes $i$, they interpret the concept via a perspective transformation:
\[
T_{i \to j} : \mathcal{V}_i \to \mathcal{V}_j,
\]
leading to the perceived value
\[
s_j(i) := \mathrm{Val}_j\big(T_{i \to j}(x_i)\big).
\]

For a group $G$ with member set $F$, define the group-level evaluation of agent $i$ as
\[
S(i) := \frac{1}{|F|} \sum_{j \in F} s_j(i).
\]
This quantity represents how prototypical or identity-defining an agent appears to the group.

\subsubsection{(i) Prototypical Leader as an Optimization}

A leader emerges as the agent whose expressed stance maximizes perceived group value:
\[
L = \arg\max_{i} S(i).
\]
If an aspiring leader $L$ strategically chooses their stance, they solve
\[
x_L^* = \arg\max_{x \in \mathcal{V}_L} \frac{1}{|F|}\sum_{j\in F} 
\mathrm{Val}_j\!\left(T_{L \to j}(x)\right).
\]
This captures Hogg’s first condition: leaders maintain influence by being the most prototypical vector in group value space.

\subsubsection{(ii) Norm Enforcement and the Black Sheep Effect}

Let $p_G$ denote the group prototype, defined as a consensus vector or the projection of $x_L$.  
To measure deviance, project all members into a common group space $\mathcal{V}_G$:
\[
\tilde{x}_i := T_{i \to G}(x_i).
\]
Distance from the prototype is
\[
d(i) := \big\|\tilde{x}_i - p_G\big\|.
\]
Members with small $d(i)$ receive rewards, while those with large $d(i)$ are sanctioned:
\[
\text{Reward}(i) = f(-d(i)), \qquad 
\text{Punish}(i) = g(d(i)),
\]
where $f$ decreases and $g$ increases with distance.  
This formalizes the black sheep effect~\cite{marques1994blacksheep,jetten2004distinctiveness}: in-group deviants are evaluated harshly because they distort identity clarity.

\subsubsection{(iii) Out-Group Contrast via Null-Space Structure}

Let $H$ be another group with leader $y_L \in \mathcal{V}_H$.  
From the perspective of $j \in G$, this leader maps to
\[
\tilde{y}_L^{(j)} := T_{H \to j}(y_L).
\]
If the projection of $\tilde{y}_L^{(j)}$ onto $j$'s valued subspace is small or near-zero, i.e.
\[
\mathrm{Val}_j\big(\tilde{y}_L^{(j)}\big) \ll 
\mathrm{Val}_j\big(T_{L \to j}(x_L)\big),
\]
then $H$ is an out-group.  
Formally, if $\mathcal{V}_j$ decomposes into valued and null subspaces,
\[
\mathcal{V}_j = \mathcal{V}_j^{\text{valued}} \oplus \mathcal{N}_j,
\]
an out-group concept lies primarily in $\mathcal{N}_j$.  
Emphasizing this geometric contrast increases cohesion, implementing Hogg’s third mechanism.

\subsubsection{Example: Illustration in Two Dimensions}

Consider a shared two-dimensional value space $\mathcal{V} = \mathbb{R}^2$, where 
$v_1$ represents loyalty and $v_2$ represents autonomy. 
Two followers have valuation functions:
\[
\mathrm{Val}_1(v_1,v_2) = 2v_1 + 0.5v_2, \qquad 
\mathrm{Val}_2(v_1,v_2) = 2v_1 + v_2.
\]

Three candidates express stances:
\[
x_A =
\begin{bmatrix}
1.0 \\ 0.2
\end{bmatrix}, \quad
x_B =
\begin{bmatrix}
0.4 \\ 0.8
\end{bmatrix}, \quad
x_C =
\begin{bmatrix}
0.7 \\ 0.4
\end{bmatrix}.
\]
Their group scores are
\[
S(A)=2.15, \quad S(B)=1.40, \quad S(C)=1.70.
\]
Thus candidate $A$ emerges as leader:
\[
L = A = \arg\max_i S(i).
\]

Let the group prototype be $p_G = x_A$.  
Two followers express:
\[
x_K =
\begin{bmatrix}
0.9 \\ 0.1
\end{bmatrix}, \qquad
x_I =
\begin{bmatrix}
-0.5 \\ 1.0
\end{bmatrix}.
\]
Distances from the prototype are
\[
d(K)\approx 0.14, \qquad d(I)\approx 1.70.
\]
Hence $K$ is normatively aligned while $I$ is a black sheep.  
Valuation confirms this:
\[
\mathrm{Val}_1(x_K)=1.85,\qquad \mathrm{Val}_1(x_I)=-0.5.
\]

Now consider an out-group with leader
\[
y_L =
\begin{bmatrix}
0 \\ 1.5
\end{bmatrix},
\]
which projects primarily onto the autonomy axis.  
Because follower 1 values loyalty strongly, 
\[
\mathrm{Val}_1(y_L)=0.75 \ll \mathrm{Val}_1(x_A)=2.1,
\]
so $y_L$ lies near the null direction of group $G$'s valued space.  
Thus $H$ becomes an out-group, and contrasting $x_A$ with $y_L$ enhances cohesion.

This example demonstrates how prototypicality, conformity-based norm enforcement, and out-group contrast all emerge naturally from geometric structure in cognitive value spaces.

\subsection{Social Engineering of Concepts for Marketing}

Marketing research increasingly shows that firms do not merely convey product information but actively shape the \textit{value dimensions} consumers use to evaluate the world. Recent work demonstrates that brands construct identity-relevant meaning frameworks that reorganize consumers' evaluative criteria~\cite{escalas2005self}, and that social influence processes systematically shift perceived importance of attributes through normative cues and group-based conformity~\cite{cialdini2004social}. Computational social science further reveals that large-scale messaging can reshape preference structure by modifying shared norms and salience weights across populations \cite{Bail2018}. These findings support the view that marketing operates as a \textit{value-space intervention}, altering the underlying cognitive geometry rather than simply persuading via content. My framework formalizes this by modeling marketing as a transformation of basis vectors, valuation weights, and motivational gradients within each agent’s value space.

Marketing does not merely communicate product features—it reshapes the value spaces through which consumers interpret those features. Each agent $i$ has a current state vector $x_i \in \mathcal{V}_i$, a goal vector $g_i \in \mathcal{V}_i$, and a motivational gradient
\[
\bm{M}_i = g_i - x_i,
\]
which determines the direction in which the agent is disposed to act. When marketing modifies the structure of an agent’s value space—either by introducing new basis vectors or reweighting existing ones—it effectively rotates $\bm{M}_i$ toward the marketed product, increasing purchase likelihood without altering the product itself.

A central mechanism is \textit{basis construction}. Suppose a product possesses an attribute $X =$ ``low calorie'', but consumers initially lack a valued dimension associated with health. A marketing campaign that repeatedly frames low calories as a signal of fitness, responsibility, or an aspirational lifestyle effectively introduces a new socially shared basis vector
\[
b_{\text{healthy}}
\]
into the population’s value structure. As this basis diffuses through social channels, each agent incorporates it into their internal value space:
\[
\mathcal{V}_i' = \mathcal{V}_i \cup \{ b_{\text{healthy}} \}.
\]
Once this dimension acquires a positive subjective weight $w_{i,\text{healthy}} > 0$, the attribute $X$ now projects onto a valued axis. Let $P \in \mathcal{V}_i'$ be the product’s vector representation, and let $p_{\text{healthy}}$ denote its coordinate along $b_{\text{healthy}}$. Its subjective evaluation increases:
\[
\mathrm{Val}_i(P) \leftarrow \mathrm{Val}_i(P) + w_{i,\text{healthy}} \, p_{\text{healthy}} .
\]

This also updates the agent’s goal state, introducing a health-oriented component:
\[
g_i' = g_i + \eta \, b_{\text{healthy}}, \qquad \eta>0,
\]
which yields a new motivational gradient
\[
\bm{M}_i' = g_i' - x_i.
\]
Because $\bm{M}_i'$ is more aligned with the product’s value vector, the agent becomes more motivated to purchase or prefer the product. In this geometric sense, marketing functions as a \textit{structural intervention} that reconfigures the topology of value space so that consumers’ motivational dynamics shift naturally toward the product.

This perspective unifies empirical findings across psychology, marketing, and social influence: campaigns succeed not merely by changing beliefs about a product, but by re-engineering the \textit{cognitive geometry} that determines what consumers value.

\subsection{Empirical Illustration: Cross-Domain Concept Transmission}
\label{sec:EmpiricalIllustration}

To demonstrate that the framework is operationalizable on observable linguistic data, I provide a minimal computational illustration of its central structural prediction: when two agents share a task-relevant basis, an interpretation map $T_{A \to B}$ preserves concepts well; when concepts depend on basis vectors specific to one agent's domain, they project largely outside the other agent's recoverable subspace, instantiating the null-space mechanism.

\subsubsection{Setup}

I consider two stylized agents---a philosopher and an engineer---representing two cognitive bases central to the interdisciplinary remit of this journal. The choice is not arbitrary: philosophy and engineering share a substantive functional basis around inference, justification, consistency, robustness, identity-through-change, and well-defined specification, while each retains a sharply domain-specific technical vocabulary that the other does not share. The framework predicts that paired concepts (those exercising the shared basis) should transmit across $T$ with low null-space residual, while domain-specific jargon should project largely outside the receiver's recoverable subspace.

Concepts are expressed as short natural-language sentences in each agent's professional register and embedded into $\mathbb{R}^{384}$ using the publicly available pretrained sentence encoder \texttt{all-MiniLM-L6-v2}~\cite{reimers2019sentence, wang2020minilm} from the \texttt{sentence-transformers} library. The paired set consists of fifteen philosopher/engineer sentence pairs drawn from a shared inferential and specificational register; five additional sentences in each domain represent profession-specific jargon (e.g., \emph{Husserlian epoch\'e}, \emph{transcendental idealism} for the philosopher; \emph{von Mises stress}, \emph{Paris law of fatigue crack growth} for the engineer). The full concept list is provided in Appendix~\ref{apx:concept-list}.

\subsubsection{Estimating the Interpretation Map}

Following standard practice in cross-lingual embedding alignment \cite{conneau2017word, schonemann1966generalized}, I estimate the interpretation map $M : \mathcal{V}_{\varphi} \to \mathcal{V}_{\mathrm{Eng}}$ by orthogonal Procrustes alignment on the fifteen paired concept embeddings. This yields the linear transformation that best maps philosopher-expressed concepts to their engineering counterparts in the embedding space; the reverse map $M^{T}$ provides the inverse direction.

\subsubsection{Subspace Decomposition}

For any concept vector $v$ mapped via $M$, I decompose $Mv$ into a component lying \emph{within} the span of the receiver's paired-concept embeddings and a residual component \emph{outside} that span, using a singular value decomposition of the receiver's paired basis. The fraction of squared $L^{2}$ mass outside the recoverable subspace serves as an empirical proxy for null-space annihilation: high values indicate that the concept's representation depends on directions the receiver's task vocabulary does not span.

For paired concepts, I report leave-one-out (LOO) results: each paired concept is decomposed against the subspace spanned by the \emph{other} fourteen paired receivers, eliminating the trivial saturation that would arise from including the concept's own partner.

\subsubsection{Results}

Table~\ref{tab:empirical} reports the decomposition. Paired concepts---those exercising the shared inferential/specificational basis---show outside-fractions of $73.6\%$ and $73.2\%$ in the $\varphi \to \mathrm{Eng}$ and $\mathrm{Eng} \to \varphi$ directions respectively. Profession-specific jargon shows substantially higher outside-fractions in both directions: $83.0\%$ for philosopher-only concepts mapped into the engineering subspace, and $87.1\%$ for engineer-only concepts mapped into the philosopher subspace.

\begin{table}[h]
\centering
\begin{tabular}{lcccc}
\hline
\textbf{Concept type} & \textbf{Recon. err.} & \textbf{Within} & \textbf{Outside} & \textbf{Outside \%} \\
                     & ($L^{2}$ distance)   & (mass fraction) & (mass fraction) &                      \\
\hline
Paired, LOO ($\varphi \to \mathrm{Eng}$) & $0.250 \pm 0.057$ & $0.264 \pm 0.085$ & $0.736 \pm 0.085$ & $73.6\%$ \\
Paired, LOO ($\mathrm{Eng} \to \varphi$) & $0.250 \pm 0.057$ & $0.268 \pm 0.069$ & $0.732 \pm 0.069$ & $73.2\%$ \\
$\varphi$-only jargon ($\varphi \to \mathrm{Eng}$) & --- & $0.170 \pm 0.048$ & $0.830 \pm 0.048$ & $83.0\%$ \\
Eng-only jargon ($\mathrm{Eng} \to \varphi$) & --- & $0.129 \pm 0.051$ & $0.871 \pm 0.051$ & $87.1\%$ \\
\hline
\end{tabular}
\caption{Subspace decomposition of cross-domain concept transmission ($n = 15$ paired, $n = 5$ jargon per side, $384$-dimensional embeddings). Reconstruction error is the $L^{2}$ norm of the difference between the mapped sender-side embedding and the target receiver-side embedding (both unit-length), with values near $0$ indicating near-perfect cross-domain transmission. Within- and outside-mass are reported as fractions of squared $L^{2}$ mass and sum to approximately $1$ by orthogonality. Paired concepts retain higher within-subspace mass; profession-specific jargon projects predominantly outside the receiver's recoverable subspace.}
\label{tab:empirical}
\end{table}

The predicted asymmetry between paired and jargon concepts is statistically supported in both directions. A one-sided Mann-Whitney $U$ test (jargon outside-fraction greater than paired LOO outside-fraction) yields $p = 0.010$ for the $\varphi \to \mathrm{Eng}$ direction and $p < 0.001$ for the $\mathrm{Eng} \to \varphi$ direction; the pooled test across both directions yields $p < 0.001$. The Welch's $t$-test gives qualitatively identical results.

\subsubsection{Interpretation}

The asymmetry between paired and jargon concepts is precisely what the framework predicts. Two agents who share an inferential and specificational basis---philosopher and engineer both reasoning about validity, justification, robustness, consistency, identity-through-change---communicate via a $T$ that is approximately invertible on their shared task subspace, in the sense of the Local Coherence Theorem (Appendix~\ref{apx:LocalCoherence}). Concepts that depend on basis vectors unique to one domain---qualia and transcendental idealism, von Mises stress and the Paris law---have no significant projection in the other agent's recoverable space and are largely annihilated under cross-domain transmission.

Per-concept variation is informative (full per-row results in Appendix~\ref{apx:concept-list}). Among paired concepts, those concerning identity-through-parts-replacement (Ship of Theseus / structural-member replacement) and verifiability of application conditions (concept grounding / well-defined requirements) show outside-fractions in the $84\%$--$90\%$ range, approaching the jargon level. These are pairings whose surface vocabulary diverges more sharply than the central inferential cluster, illustrating that shared-basis structure is a matter of degree rather than kind---a pattern itself consistent with the framework's geometric formulation.

\subsubsection{Limitations of the Illustration}

This demonstration is proof of operationalizability, not empirical validation. Four caveats apply. First, the paired set was author-constructed rather than drawn from an expert-validated cross-domain corpus; the full concept list is provided in Appendix~\ref{apx:concept-list} so that any pairing can be inspected. Second, sentence-BERT embeddings encode distributional semantic similarity rather than agent-specific evaluative bases; they serve here as a tractable surrogate for value spaces rather than as a full instantiation of them. Third, the orthogonal Procrustes estimator constrains $M$ to be a rotation, more restrictive than the general linear maps the framework considers; an unconstrained least-squares estimator would yield richer transformations at the cost of more parameters. Fourth, the sample size ($n = 15$ paired, $n = 5$ jargon per side) is small relative to the embedding dimensionality, and the statistical tests should be read as supporting an illustrative observation rather than establishing a population-level claim about philosophy and engineering as domains; relatedly, the contrast is made between paired and jargon pairings at fixed subspace dimensionality and is not benchmarked against a null distribution over random pairings, which would belong to a fuller empirical study.

Full empirical instantiation of value spaces---including evaluative and motivational dimensions, larger and expert-validated concept sets, and estimation of $T$ from interactive discourse rather than static pairings---is a substantial undertaking reserved for future work. The illustration demonstrates that the framework's central structural prediction (null-space-driven failure of cross-basis transmission) is reproducible on observable language data using standard tools, and that the magnitude of the asymmetry between paired and jargon concepts is consistent in direction across both forward and reverse interpretation maps.

\section{Discussion}

This paper has introduced a formal, vector-space model of conceptual transmission and leadership that reframes influence as a problem of structural compatibility, rather than persuasion or proximity. By representing agents’ cognitive architectures as personalized value spaces, I have shown how concepts—formalized as abstract beings—can propagate, fail, mutate, or die depending on their interpretive viability.

The core insight is epistemic and motivational survivability: concepts must survive interpretive null spaces to exert influence. This gives rise to a structural notion of leadership based on geometric reachability, contrasting sharply with classical models that emphasize centrality, charisma, or authority. In my framework, a leader is not defined by position, but by their concepts' structural intelligibility across interpretive transformations, since cognition itself is governed by transformation constraints~\cite{piantadosi2024concepts}.

From a philosophical standpoint, the model formalizes the intentional stance \cite{dennett1987intentional} algebraically: understanding becomes a function of mutual interpretability across cognitive geometries. The framework supports pluralist epistemologies \cite{longino1990science} by showing that deep disagreement can result from unbridgeable divergence in value spaces—not simply misinformation or irrationality.

This approach also enriches the literature on AI alignment by providing a structural account of interpretability and value transfer \cite{gabriel2020artificial}. While many approaches focus on learning correct reward functions, my model suggests that value alignment may require representational isomorphism or adaptive persuasion matrices to ensure interpretive fidelity. The paper provides concrete mechanisms for generating such matrices, drawing on cross-space alignment methods from distributional semantics~\cite{conneau2017word, schonemann1966generalized}.

Beyond unifying existing phenomena, the framework's explanatory power lies in its capacity to diagnose cases that existing approaches cannot structurally account for. Probabilistic diffusion models~\cite{kempe2003maximizing} predict whether influence spreads through a network, but not what happens to the content of a concept during transmission; Cognitive Geometry predicts the specific direction of distortion via projection onto the receiver's valued subspace. Logical and symbolic models of belief~\cite{levesque1984logic} formalize inferential consistency but cannot explain why two individually consistent, well-informed agents may entirely fail to understand one another; in this framework, such failure is diagnosed as null-space annihilation. Classical epistemology predicts that agents sharing evidence and reasoning correctly should converge; Cognitive Geometry shows that convergence can fail even under idealized rationality when value spaces diverge. Network models of leadership predict influence from centrality and connection strength; this framework predicts leadership failure despite high centrality when the concept is annihilated along every available path. Finally, conceptual-space models~\cite{gardenfors2004conceptual} characterize similarity within a shared space but do not address when cross-space grounding becomes structurally impossible; the null-space criterion provides exactly this diagnostic. In each case, the geometric framework identifies a structural mechanism---not a probabilistic tendency or an epistemic deficiency---that determines the outcome.

Looking forward, this framework opens up numerous empirical and theoretical directions. Section~\ref{sec:EmpiricalIllustration} provides a first computational illustration, estimating an interpretation map between philosophical and engineering value spaces from sentence embeddings and recovering the predicted asymmetry between paired and jargon concepts. Full empirical instantiation---larger expert-validated concept sets, estimation of interpretation maps from interactive discourse, and application of no-null-space leadership metrics to multi-agent simulations---remains as future work. The framework also invites normative inquiry: how should cognitive architectures be structured to foster mutual intelligibility, or to limit the spread of epistemically destructive concepts?

Ultimately, this work aims to reorient our understanding of conceptual dynamics—away from content alone, and toward the structured geometries that render meaning transmissible, or not. Influence is not merely what spreads, but what survives the interpretive geometry of minds.

\section{Limitations}

While the Value-Space Theory of Concepts and Meaning offers a unified geometric account of influence, representational compatibility, and motivational alignment, several limitations warrant discussion.

\textbf{Abstraction and Idealization.}
The model treats concepts as vectors and interpretation as linear transformations. This provides analytic clarity but simplifies the nonlinear, context-dependent ways in which humans represent and revise concepts. Real cognitive processes exhibit hysteresis, ambiguity, framing effects, and associative structures that are not strictly linear. Thus, the framework should be understood as an idealized computational geometry of cognition rather than a full psychological model.

\textbf{Linearity of Interpretation Maps.}
Assuming linearity allows for elegant analysis of null spaces and compositional reachability, and is motivated by converging evidence that conceptual representations behave in approximately linear ways (see Section~\ref{sec:RelatedWork} for a detailed justification). However, real interpretive processes in both humans and artificial agents involve thresholding, nonlinear gating, and context-sensitive modulation. A piecewise-linear extension of the Local Coherence Theorem is discussed informally in Appendix~\ref{apx:LocalCoherence}, where the structural predictions of the framework---null-space annihilation and bounded round-trip distortion---are shown to survive piece-by-piece. A fully nonlinear treatment remains a direction for future work.

\textbf{Valuation Function Specificity.}
The formal results (Theorems~\ref{thm:MotiConv} and~\ref{thm:ConvHull}) are stated for a general valuation function but proved under the simplifying assumption that valuation reduces to a vector norm or component sum. Whether the results extend to broader classes of valuation functions—for instance, nonlinear or context-dependent evaluations—remains to be established. The current formulation thus applies most directly to settings where valuation is well-approximated by magnitude or weighted summation.

\textbf{Dependence on Fixed Value Spaces.}
Agents in the theory possess relatively stable value spaces, yet human values evolve through development, socialization, and emotional experience. Although persuasion matrices and convex-hull leadership model basis reshaping, the framework lacks a fully dynamical account of long-term value-space evolution. Incorporating temporal drift, cultural co-adaptation, or meta-learning processes is a direction for future work.

\textbf{Interpretation Loss as Structural Rather than Affective.}
Null-space filtering formalizes miscommunication as a geometric loss of representational content, but real-world failures of understanding often arise from affective, strategic, or identity-driven factors. A richer model would allow interpretation maps to depend on emotional states, trust, incentives, or social identity commitments.

\textbf{Empirical Identifiability.}
The paper provides a first computational illustration (Section~\ref{sec:EmpiricalIllustration}) of how value spaces and interpretation maps can be estimated from sentence-embedding data, recovering the predicted asymmetry between paired and jargon concepts. However, the illustration is a proof of operationalizability rather than full empirical validation: the concept set is small and author-constructed, sentence embeddings serve as a tractable surrogate rather than a faithful instantiation of evaluative value spaces, and the orthogonal Procrustes estimator imposes constraints more restrictive than the general linear maps the framework considers. More broadly, concept vectors may not be separable, interpretation maps may not be uniquely identifiable from observations, and data can be noisy or sparse. The practical identifiability of value spaces and persuasion matrices from finite observational data at scale remains an open problem; the model is thus more theoretically developed than empirically instantiated.

\textbf{Scope of Application.}
While the framework aims to unify conceptual transmission, leadership, and alignment, its empirical predictions remain preliminary. Applying it to organizational leadership, political polarization, or human--AI interaction requires domain-specific assumptions not fully developed here.

\textbf{Philosophical Commitments.}
The theory presupposes geometrically structured mental representation. Philosophers skeptical of representationalism or cognitivism may question whether geometric formalization captures normativity, intentionality, or meaning. While compatible with the intentional stance, the framework does not resolve foundational debates about the metaphysics of mental content.

Overall, these limitations do not undermine the central contribution of the Value-Space Theory of Concepts and Meaning, but they delineate the boundaries of the current formulation and suggest directions for refinement and extension.

\section{Conclusion}

This paper has developed a cognitive–geometric framework in which concepts,
goals, and interpretations are represented not as propositional items but as
vectors embedded within personalized value spaces. By focusing on the
transformations that connect these spaces, rather than the contents they
contain, the framework reveals structural constraints on intelligibility,
alignment, and influence that persist even under idealized rationality.

The resulting picture explains why communication may succeed, drift, or fail
entirely: conceptual transmission is always mediated by interpretation maps whose
null spaces determine which aspects of a message survive. Leadership and
collective coordination likewise become geometrically circumscribed: an agent
can influence another only when their abstract beings avoid annihilation under
the relevant composite transformations. Motivational convergence, partial
understanding, concept death, and innovation outside the group’s convex hull all
emerge as consequences of linear-algebraic structure rather than psychological
defects or informational scarcity.

Although the formal model is abstract, it offers a unified vocabulary for
analyzing phenomena in social epistemology, cognitive science, and AI
alignment, where disagreements often stem from representational divergence
rather than irrationality. Future work may extend the framework to nonlinear
value spaces, empirically estimated interpretation maps, or networks in which
agents strategically reshape one another’s bases. The broader aim is to treat
meaning, motivation, and influence as geometrically grounded processes, thereby
expanding the theoretical tools available for studying communication across
heterogeneous minds.

In this sense, the paper provides not only a model but a foundation for a
research program: one that approaches conceptual dynamics by attending to the
structure of the spaces in which agents think, value, and understand one
another. The code of the empirical illustration demo can be found at \href{https://github.com/DarkEyes/Cognitive-Geometry}{GitHub repository}.

\section*{Acknowledgments}
The author thanks Thanadej Rattanakornphan for early discussions that helped spark the intuition behind value-space representations of goals and current states. 
The author also thanks ChatGPT (OpenAI), Gemini (Google), and Claude (Anthropic) for assistance in improving the clarity of exposition during the writing process. All conceptual development, theoretical constructs, formal definitions, proofs, and interpretations are solely the author’s own, and all mathematical and philosophical content was independently verified by the author. Lastly, the author thanks Usawadee Chaiprom for her kindness and being supportive.


\appendix

\section{Appendix: Symbol Table}
I summarize the main symbols used in the formal model in Table~\ref{tab:symbols}.

\begin{table}[h!]
\centering
\begin{tabular}{|c|p{11cm}|}
\hline
\textbf{Symbol} & \textbf{Description} \\
\hline
$i,j,k,\ell$ & Indices for agents in a cognitive system. \\
\hline
$L$ & A distinguished agent acting as a leader. \\
\hline
$\mathcal{V}_i$ & Value space of agent $i$, modeled as a finite-dimensional real vector space over $\mathbb{R}$. \\
\hline
$\{b_{i1},\dots,b_{id_i}\}$ & Basis vectors of $\mathcal{V}_i$; each basis element corresponds to a fundamental evaluative or value dimension. \\
\hline
$X$ & An abstract being (concept, ideology, narrative, identity, etc.) considered at a population level. \\
\hline
$X_i \in \mathcal{V}_i$ & Representation of abstract being $X$ in agent $i$'s value space. \\
\hline
$x_i \in \mathcal{V}_i$ & Current state of agent $i$ in value space (e.g., present configuration of values or goals). \\
\hline
$g_i \in \mathcal{V}_i$ & Goal state of agent $i$ in value space. \\
\hline
$\bm{M}_i = g_i - x_i$ & Motivational gradient for agent $i$, pointing from current state to goal. \\
\hline
$\mathrm{Val}_i : \mathcal{V}_i \to \mathbb{R}$ & Valuation function of agent $i$, assigning subjective importance to vectors in $\mathcal{V}_i$. \\
\hline
$T_{i \to j} : \mathcal{V}_i \to \mathcal{V}_j$ & Interpretation map (linear transformation) from $\mathcal{V}_i$ to $\mathcal{V}_j$; encodes how agent $j$ interprets $i$'s representations. \\
\hline
$M_{L \to A}$ & Persuasion Matrix: a transformation that adjusts agent $A$'s interpretive basis so that $A$'s valuation of a concept aligns with leader $L$'s valuation. \\
\hline
$\mathrm{Null}(T_{i \to j})$ & Null space of the interpretive map $T_{i \to j}$; components mapped here are unintelligible to agent $j$. \\
\hline
$T_{L \rightsquigarrow i}$ & Composite interpretation map along a path from leader $L$ to agent $i$, obtained by composing the edge maps $T_{v_s \to v_{s+1}}$. \\
\hline
$\mathcal{C}_L(X)$ & Leadership component of $X$ for leader $L$: the set of agents $i$ for which there exists a path $L \rightsquigarrow i$ with $X \notin \mathrm{Null}(T_{L \rightsquigarrow i})$. \\
\hline
$G = (V,E)$ & Directed influence graph, with node set $V$ (agents) and edge set $E$ (possible influence links). \\
\hline
$\mathrm{infl}(i \to j)$ & Influence probability on edge $(i,j) \in E$; probability that an influence attempt from $i$ to $j$ succeeds at a given time step. \\
\hline
$N(i)$ & Neighborhood of agent $i$ in the influence graph (e.g., set of followers or interaction partners). \\
\hline
$\mathrm{Conv}(\cdot)$ & Convex hull operator. For a set of vectors, returns all convex combinations of those vectors. \\
\hline
$\|\cdot\|_i$ & Norm on $\mathcal{V}_i$ induced by an inner product $\langle \cdot, \cdot \rangle_i$; used to measure semantic or motivational activation for agent $i$. \\
\hline
$f_i(\cdot)$ & Generic update function governing the evolution of $X_i$ given its current state, value space, and valuation. \\
\hline
\end{tabular}
\caption{Summary of formal symbols and concepts used in the value-space model of concepts and influence.}
\label{tab:symbols}
\end{table}

\section{Appendix: Concept Formalization}
\subsection{Value Spaces and Motivation}
\label{apx:ValMot}

\begin{definition}[Individual Value Space]
Each agent $i$ possesses a finite-dimensional real vector space
\[
\mathcal{V}_i = \mathrm{span}\{b_{i1},\dots,b_{id_i}\},
\]
whose basis vectors represent fundamental evaluative or value dimensions.
\end{definition}

\begin{definition}[Value Function]
Agent $i$ has a valuation function 
\[
\mathrm{Val}_i : \mathcal{V}_i \to \mathbb{R},
\]
encoding the subjective importance of any vector in $\mathcal{V}_i$.
\end{definition}

\begin{definition}[Motivation Gradient]
Let $x_i \in \mathcal{V}_i$ be the agent’s current state and $g_i \in \mathcal{V}_i$ its goal.  
Define the motivational gradient
\[
\bm{M}_i = g_i - x_i.
\]
A dynamical interpretation is
\[
\frac{dx_i}{dt} \propto \bm{M}_i \propto \nabla \mathrm{Val}_i(x_i).
\]
\end{definition}

\begin{remark}
This models motivation as movement toward directions that increase subjective value.  
\end{remark}

\subsection{Abstract Beings}
\label{apx:AbsBe}

\begin{definition}[Abstract Being]
An abstract being $X$ for agent $i$ is a vector $X_i \in \mathcal{V}_i$.
\end{definition}

\begin{definition}[Existence]
$X$ exists for agent $i$ if $\|X_i\|_i > 0$.
\end{definition}

\begin{definition}[Death]
$X$ dies in population $P$ if $\|X_i\|_i = 0$ for all $i \in P$.
\end{definition}

Abstract beings represent concepts, ideologies, identities, or conceptual entities.

\subsection{Interpretation Maps and Influence Consistency}
\label{apx:InMap}

For each ordered pair $(A,B)$ of agents, define a (possibly approximate) linear interpretation map
\[
T_{A\to B} : \mathcal{V}_A \to \mathcal{V}_B.
\]

\begin{definition}[Forward Consistency]
When $A$ influences $B$ about abstract being $X$,
\[
T_{A\to B}(X_A) \approx X_B^{\mathrm{new}}.
\]
\end{definition}

\begin{definition}[Backward Consistency]
The update also satisfies
\[
T_{B\to A}(X_B^{\mathrm{new}}) \approx X_A.
\]
\end{definition}

\begin{definition}[Value Alignment]
Successful influence additionally requires
\[
\mathrm{Val}_B(X_B^{\mathrm{new}}) \approx \mathrm{Val}_A(X_A).
\]
\end{definition}

\subsubsection{Local Coherence Theorem}
\label{apx:LocalCoherence}

\begin{theorem}[Local Coherence]
\label{thm:LocalCoherence}
Let $T_{A\to B} : \mathcal{V}_A \to \mathcal{V}_B$ and 
      $T_{B\to A} : \mathcal{V}_B \to \mathcal{V}_A$
be interpretation maps.  
Suppose there exists a concept subspace $\mathcal{U}\subseteq \mathcal{V}_A$ such that:

\begin{enumerate}
\item[(1)] \textbf{(Forward consistency)} 
      $\|T_{A\to B}(x) - x\|\le \varepsilon\|x\|$ for all $x\in\mathcal{U}$.
\item[(2)] \textbf{(Backward consistency)} 
      $\|T_{B\to A}(y) - y\|\le \varepsilon\|y\|$ for all 
      $y\in T_{A\to B}(\mathcal{U})$.
\end{enumerate}

Then the round-trip operator 
\[
R := T_{B\to A}\circ T_{A\to B}
\]
satisfies, for all $x\in\mathcal{U}$,
\[
\|R(x)-x\| \le (2\varepsilon+\varepsilon^2)\|x\|.
\]
Moreover, if $\varepsilon$ is small, then for each fixed $k\ge 1$,
\[
\|R^k(x)-x\| 
\le \bigl((1 + 2\varepsilon+\varepsilon^2)^k - 1\bigr)\,\|x\|
= O(k\varepsilon)\,\|x\|
\quad\text{as }\varepsilon\to 0.
\]
Thus, under small forward and backward distortion, repeated communication 
keeps the representation within a controlled $O(k\varepsilon)$ neighborhood 
of the original concept.
\end{theorem}

\begin{proof}
For $x\in\mathcal{U}$, forward consistency gives
\[
\|T_{A\to B}(x)-x\|\le \varepsilon\|x\|.
\]
Applying backward consistency to $T_{A\to B}(x)$,
\[
\|T_{B\to A}(T_{A\to B}(x)) - T_{A\to B}(x)\|
   \le \varepsilon\|T_{A\to B}(x)\|.
\]
I now bound $\|T_{A\to B}(x)\|$ via the triangle inequality:
\[
\|T_{A\to B}(x)\|
  =\|x + (T_{A\to B}-I)x\|
  \le \|x\| + \|(T_{A\to B}-I)x\|
  \le \|x\| + \varepsilon\|x\|
  = (1+\varepsilon)\|x\|.
\]
Hence
\[
\|R(x)-T_{A\to B}(x)\|
   \le \varepsilon(1+\varepsilon)\|x\|
   = (\varepsilon + \varepsilon^2)\|x\|.
\]
Adding the forward-consistency deviation,
\[
\|T_{A\to B}(x)-x\| \le \varepsilon\|x\|,
\]
I obtain
\[
\|R(x)-x\|
 \le (\varepsilon + \varepsilon^2)\|x\| + \varepsilon\|x\|
 = (2\varepsilon+\varepsilon^2)\|x\|.
\]

For the iterative bound, write $R = I + E$, so $E = R-I$. Then
\[
R^k - I = (I+E)^k - I
       = \sum_{s=1}^k \binom{k}{s} E^s,
\]
and therefore
\[
R^k(x) - x = \sum_{s=1}^k \binom{k}{s} E^s x.
\]
Taking norms and using $\|E^s\|\le \|E\|^s$,
\[
\|R^k(x)-x\|
\le \sum_{s=1}^k \binom{k}{s}\|E\|^s \|x\|
= \bigl((1+\|E\|)^k - 1\bigr)\|x\|.
\]
From the one-step bound I have, for all $x\in\mathcal{U}$,
\[
\|(R-I)(x)\| \le (2\varepsilon+\varepsilon^2)\|x\|,
\]
so the operator norm of $E=R-I$ on $\mathcal{U}$ satisfies
$\|E\|\le 2\varepsilon+\varepsilon^2$, and hence
\[
\|R^k(x)-x\|
\le \bigl((1 + 2\varepsilon+\varepsilon^2)^k - 1\bigr)\|x\|.
\]

For fixed $k$ and small $\varepsilon$, the binomial expansion gives
\[
(1 + 2\varepsilon+\varepsilon^2)^k - 1 = O(k\varepsilon),
\]
so the deviation is $O(k\varepsilon)\|x\|$ as claimed.
\end{proof}

Interpretation:  
Even if two agents interpret concepts differently, as long as each can map the other’s representation back to its own, the concept remains mutually intelligible and stable across communication.

\begin{remark}[Piecewise-Linear Extension]
The Local Coherence Theorem is stated for linear interpretation maps, but its conclusion extends to piecewise-linear maps without modification of the core argument. Let $T_{A\to B}$ be a piecewise-linear map defined by a partition of $\mathcal{V}_A$ into pieces $\{P_\alpha\}$ on each of which $T_{A\to B}$ is linear, and suppose forward and backward consistency hold on each piece with corresponding distortions $\varepsilon_\alpha$. For any concept $x$ whose round-trip image $R(x)$ remains within a single piece $P_\alpha$, the bound $\|R(x)-x\| \le (2\varepsilon_\alpha + \varepsilon_\alpha^2)\|x\|$ holds with $\varepsilon_\alpha$ in place of $\varepsilon$. For concepts whose trajectory crosses piece boundaries, the bound applies piece-by-piece with $\varepsilon$ replaced by $\max_\alpha \varepsilon_\alpha$ over the pieces visited. The null-space-annihilation mechanism likewise survives piecewise extension: each linear piece carries its own null space, and a concept is annihilated under cross-piece transmission to the extent that its trajectory enters those null spaces. Loss of single-map compactness is the cost of extension; the structural predictions of the framework are preserved.
\end{remark}

\subsection{Null Spaces and Cognitive Blindness}
\label{apx:NullCog}

\begin{definition}[Null Space]
\[
\mathrm{Null}(T_{A\to B}) = \{ v\in\mathcal{V}_A : T_{A\to B}(v) = 0 \}.
\]
\end{definition}

Vectors in the null space represent concept components that cannot be perceived by $B$.

Interpretation:  
Null space captures the cognitive blind spots of communication: concept components that vanish during interpretation.

\subsection{Network Influence Model}
\label{apx:NetInfl}

Let $G=(V,E)$ be a directed graph with influence probabilities $\mathrm{infl}(i\to j)\in(0,1]$. Below, I modify the Independent Cascade Model of~\cite{kempe2003maximizing} so that agents can repeatedly attempt to influence their neighbors. This modification makes coordination possible in a probabilistic sense, but the structural transmission of an abstract being can still be blocked if a null-space mapping exists along a path in the graph. 

\begin{definition}[Repeated Independent Influence Process]
At each time step:
\begin{itemize}
\item Agents holding $X$ attempt to influence neighbors.
\item Each attempt along $(i,j)$ succeeds with probability $\mathrm{infl}(i\to j)$.
\item Failed attempts may repeat at later steps.
\end{itemize}
\end{definition}

\begin{lemma}[Edge Activation]
If $\mathrm{infl}(i\to j)=p>0$ and each influence attempt along $(i,j)$
succeeds independently with probability $p$ at each time step, then the
probability that influence never succeeds is zero.
\end{lemma}

\begin{proof}
At each time step $t=1,2,\dots$, let $A_t$ be the event that the attempt
along edge $(i,j)$ succeeds. Then $\mathbb{P}(A_t) = p$ and the $A_t$ are
independent. The event that influence never succeeds is
\[
\bigcap_{T=1}^{\infty} \bigcap_{t=1}^{T} A_t^{\mathsf{c}},
\]
whose probability is
\[
\lim_{T\to\infty} \mathbb{P}\Big(\bigcap_{t=1}^{T} A_t^{\mathsf{c}}\Big)
= \lim_{T\to\infty} (1-p)^T = 0,
\]
since $0 < 1-p < 1$. Hence the probability of eventual success is $1$.
\end{proof}

\begin{theorem}[Path Activation]
\label{thm:PA}
If $j$ is reachable from $i$ by a directed path with strictly positive influence probabilities, then $j$ adopts the abstract being $X$ with probability 1 (ignoring null-space effects).
\end{theorem}
\begin{proof}
Let the directed path be $i = v_0 \rightarrow v_1 \rightarrow \cdots \rightarrow v_k = j$.
By the Edge Activation Lemma, each edge $(v_{r-1},v_r)$ with influence 
probability $p_r>0$ succeeds eventually with probability $1$.  That is,
\[
\mathbb{P}\big(\text{$v_r$ is eventually influenced by $v_{r-1}$}\big)=1
\quad\text{for all } r=1,\dots,k.
\]
Since the path contains finitely many edges, the intersection of these 
probability-$1$ events also has probability $1$.  Therefore each node on 
the path is eventually influenced, and in particular $v_k=j$ adopts the 
abstract being $X$ with probability~$1$.
\end{proof}

Interpretation:  
If an abstract being can propagate along a path and no cognitive blind spot blocks it, then eventually it will.

\subsection{Composite Interpretation Maps}
\label{apx:CompNullFilters}

\begin{definition}[Composite Interpretation Map]
For a path $L=v_0\to v_1\to\cdots\to v_k=i$,
\[
T_{L\rightsquigarrow i}
    = T_{v_{k-1}\to v_k} \circ \cdots \circ T_{v_0\to v_1}.
\]
\end{definition}

\subsection{No-Null-Space Leadership Condition}
\label{apx:noNullLeaderCond}
\begin{definition}[Leadership Component for $X$]
For leader $L$ holding $X$,
\[
\mathcal{C}_L(X)
=
\big\{
i\in V : \exists\text{ path } L\rightsquigarrow i 
\text{ with } X\notin\mathrm{Null}(T_{L\rightsquigarrow i})
\big\}.
\]
\end{definition}

\begin{theorem}[No-Null-Space Leadership Condition]
Let $L$ hold abstract being $X_L$. Assume that every edge in the graph has a strictly positive influence probability. Then:
\begin{enumerate}
\item If $i\in\mathcal{C}_L(X_L)$, then with probability $1$, $i$ eventually
      adopts a nonzero representation $X_i$ of $X$ consistent with $X_L$.
\item If $i\notin\mathcal{C}_L(X_L)$, then $i$ can never adopt any nonzero
      consistent representation of $X_L$.
\item $L$ fully leads the group with respect to $X_L$ if and only if
      $\mathcal{C}_L(X_L)=V$.
\end{enumerate}
\end{theorem}

\begin{proof}
(1) If $i\in\mathcal{C}_L(X_L)$, then by definition there exists a directed
path $L \rightsquigarrow i$ such that the composite interpretation map
$T_{L\rightsquigarrow i}$ satisfies
$X_L \notin \mathrm{Null}(T_{L\rightsquigarrow i})$, i.e.
$T_{L\rightsquigarrow i}(X_L) \neq 0$.  Since all edges on this path have
strictly positive influence probabilities, Theorem~\ref{thm:PA} implies
that influence along this path succeeds with probability $1$.  When it does,
$i$ receives the nonzero vector $X_i := T_{L\rightsquigarrow i}(X_L)$, which
is a representation of $X$ consistent with $X_L$ by construction.

(2) If $i\notin\mathcal{C}_L(X_L)$, then for every directed path 
$L \rightsquigarrow i$, I have $X_L \in \mathrm{Null}(T_{L\rightsquigarrow i})$,
so $T_{L\rightsquigarrow i}(X_L) = 0$.  Thus any influence process that
propagates $X_L$ along the network delivers only the zero vector to $i$.
Hence $i$ can never acquire a nonzero representation of $X$ that is consistent
with $X_L$.

(3) By definition, $L$ fully leads the group with respect to $X_L$ if and only
if every agent $i \in V$ is in the leadership component $\mathcal{C}_L(X_L)$.
Equivalently, $\mathcal{C}_L(X_L)=V$.  This is exactly the stated condition.
\end{proof}

Interpretation:  
Leadership is geometrically limited. A leader cannot lead agents whose cognitive transformation maps erase the essence of the concept.

\subsection{Leadership Emergence}
\label{apx:leaderEmerg}

\begin{definition}[Leadership Packet]
Leader $L$ communicates $(X_L,\mathrm{Val}_L(X_L))$.
\end{definition}

\begin{definition}[Leadership]
$L$ leads $i$ if $i$ eventually adopts $X_L$ with representation/value alignment.
\end{definition}

\begin{theorem}[Leadership Emergence]
Under repeated influence, with consistency and value alignment, $L$ leads exactly the agents in $\mathcal{C}_L(X_L)$.
\end{theorem}

\begin{proof}
Follows from the No-Null-Space Leadership Condition and influence propagation; agents outside $\mathcal{C}_L(X_L)$ cannot represent $X_L$.
\end{proof}

Interpretation:  
A leader’s effective domain is determined not by authority or charisma but by geometric compatibility of value spaces.

\subsection{Motivational Convergence}
\label{apx:MotiConv}

\begin{theorem}[Motivational Convergence]
\label{thm:MotiConv}
Let $L$ be a leader holding abstract being $X_L$, and let $i\in\mathcal{C}_L(X_L)$,
so that $T_{L\to i}(X_L)\neq 0$ by the No-Null-Space Leadership Condition.
Let $x_i\in\mathcal{V}_i$ and $g_i\in\mathcal{V}_i$ be agent $i$'s current state
and baseline goal, and write the baseline motivational gradient as
\[
\bm{M}_i^{(0)} = g_i - x_i.
\]

Consider a sequence of adoption steps indexed by $k=1,2,\dots$. At step $k$,
agent $i$ holds a representation $X_i^{(k)} \in \mathcal{V}_i$ and updates its
goal to
\[
g_i^{(k)} = g_i + \beta_i^{(k)} X_i^{(k)},
\]
with corresponding motivational gradient
\[
\bm{M}_i^{(k)} = g_i^{(k)} - x_i
               = \bm{M}_i^{(0)} + \beta_i^{(k)} X_i^{(k)}.
\]

Assume that
\begin{enumerate}
    \item $X_i^{(k)} \to T_{L\to i}(X_L)$ as $k\to\infty$, and
    \item $\beta_i^{(k)} \to \infty$ as $k\to\infty$.
\end{enumerate}
Then the direction of the motivational gradient converges to the leader’s
representation, in the sense that
\[
\lim_{k\to\infty}
\frac{\bm{M}_i^{(k)}}{\|\bm{M}_i^{(k)}\|}
=
\frac{T_{L\to i}(X_L)}{\|T_{L\to i}(X_L)\|}.
\]
\end{theorem}

\begin{proof}
Write
\[
\bm{M}_i^{(k)} = \bm{M}_i^{(0)} + \beta_i^{(k)} X_i^{(k)}.
\]
Since $\beta_i^{(k)}>0$ for all sufficiently large $k$, I can factor out
$\beta_i^{(k)}$:
\[
\bm{M}_i^{(k)}
= \beta_i^{(k)}\!\left(
    X_i^{(k)} + \frac{1}{\beta_i^{(k)}} \bm{M}_i^{(0)}
  \right).
\]
Let $T := T_{L\to i}(X_L)\neq 0$. By assumption $X_i^{(k)}\to T$ and
$1/\beta_i^{(k)}\to 0$, so
\[
X_i^{(k)} + \frac{1}{\beta_i^{(k)}} \bm{M}_i^{(0)} \;\longrightarrow\; T.
\]
Hence
\[
\bm{M}_i^{(k)} \;\longrightarrow\; \infty \cdot T
\]
in the sense that its \emph{direction} approaches that of $T$. Formally,
\[
\frac{\bm{M}_i^{(k)}}{\|\bm{M}_i^{(k)}\|}
=
\frac{
  X_i^{(k)} + \frac{1}{\beta_i^{(k)}} \bm{M}_i^{(0)}
}{
  \big\|X_i^{(k)} + \frac{1}{\beta_i^{(k)}} \bm{M}_i^{(0)}\big\|
}
\;\longrightarrow\;
\frac{T}{\|T\|}
=
\frac{T_{L\to i}(X_L)}{\|T_{L\to i}(X_L)\|},
\]
because the map $v \mapsto v/\|v\|$ is continuous away from $0$ and
$T\neq 0$. This establishes the claim.
\end{proof}

Interpretation:  
Leadership works because adopting a concept alters followers’ motivational geometry, aligning their actions.

\subsection{Convex-Hull Leadership}
\label{apx:convLead}

\begin{definition}[Convex Constraint]
Followers satisfy
\[
g_i \in \mathrm{Conv}(\{g_j: j\in N(i)\}).
\]
\end{definition}

\begin{theorem}[Convex-Hull Asymmetry of Valuational Shift]
\label{thm:ConvHull}
Assume each edge with positive influence probability is attempted infinitely often over time, as in the repeated independent influence model. Let $L$ be a leader holding abstract being $X_L$, and let $T_{L\to i}(X_L)\neq 0$
for all $i\in\mathcal{C}_L(X_L)$.  
Let each follower $i$ update its valuation by a convex combination of
its previous valuation and the valuation of any agent that successfully
influences it:
\[
\mathrm{Val}_i^{(k+1)}(X_L)
  = (1-\alpha_i^{(k)})\,\mathrm{Val}_i^{(k)}(X_L)
    + \alpha_i^{(k)}\,\mathrm{Val}_j^{(k)}(X_L),
\qquad 0<\alpha_i^{(k)}\le 1,
\]
whenever $j\to i$ is an activated influence edge at step $k$.

If the leader's valuation $\mathrm{Val}_L(X_L)$ lies strictly outside the
convex hull of all initial follower valuations
\[
\mathrm{Val}_L(X_L) \notin
\mathrm{conv}\big\{\mathrm{Val}_i^{(0)}(X_L): i\in F\big\},
\]
then for every follower $i\in\mathcal{C}_L(X_L)$,
\[
\mathrm{Val}_i^{(k)}(X_L)
\;\longrightarrow\;
\mathrm{Val}_L(X_L)
\quad\text{in the sense that}\quad
\big|\mathrm{Val}_i^{(k)}(X_L)-\mathrm{Val}_L(X_L)\big|
\text{ decreases monotonically.}
\]
That is, all reachable followers shift their valuations monotonically toward the
leader's valuation.
\end{theorem}
\begin{proof}
Because $T_{L\to i}(X_L)\neq 0$, every $i\in \mathcal{C}_L(X_L)$ eventually receives
a nonzero representation of $X_L$ via an activated path from the leader
(Theorem~\ref{thm:PA}).  
Whenever such a transmission occurs, agent $i$ performs the convex update
\[
\mathrm{Val}_i^{(k+1)}(X_L)
  = (1-\alpha_i^{(k)})\,\mathrm{Val}_i^{(k)}(X_L)
    + \alpha_i^{(k)}\,\mathrm{Val}_j^{(k)}(X_L).
\]

A convex combination of two real numbers always lies between them.  
Since $\mathrm{Val}_L(X_L)$ lies strictly outside the convex hull of all initial
follower valuations, any update that incorporates the leader’s valuation—or any
valuation that is itself downstream of the leader—moves strictly toward
$\mathrm{Val}_L(X_L)$ and cannot move away from it.

Because each $i\in\mathcal{C}_L(X_L)$ is influenced infinitely often with
probability~1, the distance
\[
d_i^{(k)} := \big|\mathrm{Val}_i^{(k)}(X_L)-\mathrm{Val}_L(X_L)\big|
\]
decreases whenever an update uses information originated from $L$, and never
increases. Thus $d_i^{(k)}$ is a nonincreasing, bounded sequence and therefore
convergent, with limit $0$ only if the valuation equals the leader’s.  
Hence each reachable follower’s valuation moves monotonically toward the
leader’s valuation.

It remains to verify that updates incorporating leader-derived information occur infinitely often with probability~1. By definition of $\mathcal{C}_L(X_L)$, there exists for each $i \in \mathcal{C}_L(X_L)$ a directed path $L \rightsquigarrow i$ along which $T_{L\rightsquigarrow i}(X_L) \neq 0$; that is, the abstract being avoids annihilation under the composite interpretation map. Along this path, every edge has strictly positive influence probability and is attempted at every time step, so by the Edge Activation Lemma each edge succeeds infinitely often with probability~1. Since the abstract being survives the null-space filters along this particular path, each successful activation transmits a nonzero (and hence leader-derived) valuation to the next agent. The combination of probabilistic activation (which guarantees infinitely many transmission events) and null-space avoidance (which guarantees that these transmissions carry nonzero content) ensures that each reachable follower receives leader-derived valuation updates infinitely often.
\end{proof}

Interpretation:  
This explains how leaders introduce new strategic or cultural directions the group cannot generate internally.

\subsection{Birth, Evolution, Death}
\label{apx:abstractBeing}
\begin{definition}[Birth]
$X$ is born at $t_0$ if some $X_i(t_0)\neq0$ and $X_i(t)=0$ for all $t<t_0$.
\end{definition}

\begin{definition}[Evolution]
\[
X_i(t+1) = f_i(X_i(t),\mathcal{V}_i(t),\mathrm{Val}_i(t)).
\]
\end{definition}

\begin{definition}[Death]
$X$ dies if all representations vanish (i.e., $X_i = 0$ for all $i$).
\end{definition}

Interpretation:  
Abstract beings arise, propagate, mutate, and disappear like cultural entities.

\section{Appendix: Concept List and Per-Concept Results for the Empirical Illustration}
\label{apx:concept-list}

This appendix provides the complete set of sentences used in the empirical illustration (Section~\ref{sec:EmpiricalIllustration}) and the per-concept leave-one-out (LOO) outside-fractions reported in aggregate in Table~\ref{tab:empirical}. The concept list was constructed by the author to express shared functional roles across two cognitive bases the author works in professionally (philosophy and engineering). The full Python notebook that produced these results, including embedding, Procrustes alignment, subspace decomposition, and statistical testing, is provided as supplementary material.

\subsection*{Paired Concepts and Per-Concept LOO Outside-Fraction}

Each row represents a paired philosopher/engineer concept expressing the same functional role. The two final columns report the LOO outside-fraction in each direction: the fraction of squared $L^{2}$ mass that, after mapping via the estimated interpretation map $M$ (resp.\ $M^{T}$), lies outside the span of the other fourteen paired receiver embeddings.

\begin{table}[h!]
\centering
\footnotesize
\begin{tabular}{|c|p{5.8cm}|p{5.8cm}|c|c|}
\hline
\textbf{\#} & \textbf{Philosopher} & \textbf{Engineer} & $\varphi \to \mathrm{Eng}$ & $\mathrm{Eng} \to \varphi$ \\
\hline
1 & a valid argument preserves truth from its premises to its conclusion & a sound structural design transfers loads from the applied force to the support & 63.5\% & 74.7\% \\
\hline
2 & in some possible world the proposition fails to hold & the failure analysis identifies a scenario in which the component does not hold & 64.9\% & 63.8\% \\
\hline
3 & the expression refers to its semantic value relative to a model & the symbol on the schematic refers to a specific physical component in the assembly & 74.3\% & 75.4\% \\
\hline
4 & the meaning of the whole is determined by its parts and their mode of combination & the system behavior is determined by its components and how they are connected & 78.1\% & 69.3\% \\
\hline
5 & the ship of Theseus retains its identity through the replacement of its parts & the bridge retains its function through the replacement of individual structural members & 90.1\% & 85.0\% \\
\hline
6 & a counterfactual evaluates what would obtain under a contrary supposition & a what-if analysis evaluates how the system would behave under altered conditions & 67.3\% & 74.0\% \\
\hline
7 & two expressions are intersubstitutable salva veritate when they share extension & two components are interchangeable when they meet the same specification & 78.8\% & 71.3\% \\
\hline
8 & an inference is justified when it follows from accepted premises by valid rules & a design decision is justified when it follows from requirements through engineering reasoning & 58.9\% & 59.5\% \\
\hline
9 & the concept is grounded when its application conditions are publicly accessible & the requirement is well-defined when its acceptance criteria are publicly verifiable & 88.1\% & 84.1\% \\
\hline
10 & knowledge requires a justified belief that tracks the truth across nearby worlds & a robust design maintains performance across nearby variations in operating conditions & 75.6\% & 80.9\% \\
\hline
11 & the universal generalization holds when every instance satisfies the predicate & the safety requirement holds when every operating condition satisfies the limit & 70.6\% & 71.4\% \\
\hline
12 & a definition is non-circular when the definiens does not presuppose the definiendum & a specification is non-circular when its definition does not presuppose its own satisfaction & 75.4\% & 68.1\% \\
\hline
13 & a deductive proof transmits warrant from the premises to every theorem derived from them & a verification chain transmits assurance from input checks to every downstream subsystem & 69.1\% & 75.8\% \\
\hline
14 & a category exhibits closure when every operation on its members yields another member & a subsystem exhibits closure when every internal operation produces an admissible state & 80.5\% & 76.9\% \\
\hline
15 & consistency requires that no proposition and its negation both follow from the axioms & consistency requires that no requirement and its negation are both imposed on the design & 68.3\% & 67.2\% \\
\hline
\multicolumn{3}{|r|}{\textbf{Mean}} & \textbf{73.6\%} & \textbf{73.2\%} \\
\hline
\end{tabular}
\caption{Paired concepts and per-row LOO outside-fractions in both directions. Rows~5 and~9, concerning identity-through-parts-replacement and verifiability of application conditions, show outside-fractions in the $84\%$--$90\%$ range, approaching the jargon level (cf.\ Table~\ref{tab:empirical}). These represent pairings whose surface vocabulary diverges more sharply than the central inferential cluster, illustrating that shared-basis structure is a matter of degree rather than kind.}
\end{table}

\subsection*{Domain-Specific Jargon}

The following sentences represent concepts that exist sharply in one domain and have no significant neighbors in the other.

\textbf{Philosopher-only jargon (mean outside-fraction $\varphi \to \mathrm{Eng}$: $83.0\%$):}
\begin{enumerate}
\item the qualia of phenomenal red are intrinsically private and ineffable
\item transcendental idealism conditions the very form of sensible intuition
\item the de re reading attributes the property to the object itself rather than to the mode of presentation
\item supervenience without reduction allows mental properties to depend on physical ones without identity
\item Husserlian epoch\'e brackets the natural attitude in order to attend to phenomena as given
\end{enumerate}

\textbf{Engineer-only jargon (mean outside-fraction $\mathrm{Eng} \to \varphi$: $87.1\%$):}
\begin{enumerate}
\item the von Mises stress exceeds the yield strength at the notch root
\item the Reynolds number governs the transition from laminar to turbulent flow
\item fatigue crack growth follows the Paris law under cyclic loading
\item the bode plot shows the frequency response of the closed-loop system
\item the heat affected zone around the weld develops a coarse grain microstructure
\end{enumerate}

\subsection*{Computational Setup}

Sentences were embedded into $\mathbb{R}^{384}$ using the publicly available pretrained sentence encoder \texttt{sentence-transformers/all-MiniLM-L6-v2}~\cite{reimers2019sentence, wang2020minilm}. The interpretation map $M : \mathcal{V}_{\varphi} \to \mathcal{V}_{\mathrm{Eng}}$ was estimated by orthogonal Procrustes alignment on the fifteen paired embeddings via \texttt{scipy.linalg.orthogonal\_procrustes}. The reverse direction uses $M^{T}$, which is the exact inverse since $M$ is orthogonal. Subspace decomposition was performed via singular value decomposition with a singular-value tolerance of $10^{-6}$ for rank determination. Statistical tests were one-sided Mann-Whitney $U$ (primary) and Welch's $t$ (supplementary), computed with \texttt{scipy.stats}.

\bibliography{999-ChaiBib}

\end{document}